\documentclass[11pt]{article}
\usepackage[left=1in,top=1in,right=1in,bottom=1in,head=0in,nofoot]{geometry}

\usepackage{setspace,url,bm,amsmath} 
\usepackage{scalerel}

\usepackage{xcolor}
 \usepackage{multirow}
 \usepackage{arydshln}
 \usepackage{mathtools}
\usepackage{titlesec} 
\titlelabel{\thetitle.\quad} 
\titleformat*{\section}{\bf\Large\center}

\usepackage{authblk}
\usepackage{float}
\usepackage{graphicx} 
\usepackage{bbm}
\usepackage{latexsym}
\usepackage{caption}
\usepackage[margin=20pt]{subcaption}
\usepackage{enumerate}
\graphicspath{ {./Graphs/} }

\usepackage{amsthm}
\usepackage{amssymb}
\usepackage{amsmath}
\usepackage{color}

\usepackage{comment}
\theoremstyle{definition}
\newtheorem{assumption}{Assumption}
\newtheorem*{theorem*}{Theorem}
\newtheorem{theorem}{Theorem}
\newtheorem*{rmk*}{remark}
\newtheorem{proposition}{Proposition}
\newtheorem{lemma}{Lemma}
\newtheorem{example}{Example}

\newtheorem{definition}{Definition}
\newtheorem{remark}{Remark}
\newtheorem{corollary}{Corollary}
\newtheorem*{corollary*}{Corollary}

\usepackage{natbib} 
\bibpunct{(}{)}{;}{a}{}{,} 

\usepackage{etoolbox} 
\apptocmd{\sloppy}{\hbadness 10000\relax}{}{} 

\usepackage{multibib}
\newcites{sec}{References}

\usepackage{color}
\usepackage{listings}
\usepackage{hyperref}
\usepackage{booktabs}

\usepackage{lscape}

\RequirePackage[normalem]{ulem}

\def \E  {\mathbb{E}}
 
\def \P{\mathbb{P}}

\newcommand{\HR}{\textup{HR}}
\newcommand{\HQ}{\textup{HQ}}
\def \I {\mathbb{I}}
\newcommand{\D}{\mathcal{D}}
\newcommand{\Q}{\mathcal{Q}}
\newcommand{\T}{\mathcal{T}} 
\newcommand{\R}{{\mathbb{R}}}
\def \var {\text{Var}}
\newcommand{\col}[1]{{\color{magenta}  #1}}

\newcommand{\indep}{\perp \!\!\! \perp}

\usepackage{algpseudocode}
\usepackage{algorithm}

\allowdisplaybreaks

\begin{document}

\singlespacing

\title{\bf  
Counterfactually Safe Reinforcement Learning} 

\author[1]{Jingyi Li}
\author[2]{Peng Wu\thanks{Corresponding author: pengwu@btbu.edu.cn.}}
\author[3]{Chengchun Shi}
\affil[1]{\small Department of Statistics and Data Science, National University of Singapore}
\affil[2]{\small School of Mathematics and Statistics, Beijing Technology and Business University}  
\affil[3]{\small Department of Statistics, London School of Economics and Political Science}  



\date{}

\maketitle

\begin{abstract}
Reinforcement learning algorithms are generally designed to maximize the expected return across a population. However, a policy that is optimal on average may be suboptimal for certain individuals, leading to potential safety concerns. To address this, we first formalize the notion of individual harm from a counterfactual perspective and define harm as the event in which a chosen action results in a strictly worse outcome than a baseline alternative. We then propose a general two-stage procedure for learning policies that maximize the expected return while accounting for individual harm. We further establish the finite-sample properties of the learned policy, derive an upper bound on its sub-optimality gap, and show that the harm rate remains well-controlled. Numerical experiments on both simulated and real-world datasets demonstrate the effectiveness of the proposed approach.   
\end{abstract}



\newpage

\onehalfspacing

\section{Introduction}
\label{sec:intro}

Reinforcement Learning (RL) has recently become an active area of research in artificial intelligence and computer science~\citep{sutton1998reinforcement, li2017deep}. 
Characterized by its versatile framework, RL has found widespread application across a variety of fields and disciplinary backgrounds, ranging from sophisticated robotics and complex control systems to strategic planning and finance~\citep{kober2013reinforcement, silver2016mastering, liu2021finrl}.  
Its significance of RL has been further heightened by the emergence of large language models~\citep{ouyang2022training}, where techniques such as Reinforcement Learning from Human Feedback (RLHF) are essential for aligning model's outputs with human values.

The primary objective of RL is to find a policy $\pi(\cdot)$, as a function that maps a given state $x$ to an action $a$, that maximizes the expected cumulative outcome within a population, assuming that larger outcomes are preferable. However, this emphasis on the expectation fails to account for individual-level heterogeneity and may result in worse outcomes for certain individuals compared with the status quo~\citep{Lei-Candes2021, MuellerPearl2023-CausalInference, Jin-etal2023, wu2024quantifying}, 
 potentially raising safety concerns, as discussed in the motivating example below.

\subsection{Motivation}  \label{sec-motivation}
We present a synthetic example to illustrate the motivation of this paper. Consider four patients, each facing a choice between two actions (surgery or medication) at 2 time steps. At $t=0$, all patients are in state $x=1$, and at $t=1$, they transition to state $x=2$. Figure  \ref{fig:intro} displays the numerical outcomes for each patient under each potential action per time step. We say the surgery (treatment) is harmful for a patient if it leads to a worse outcome than the medication (control)~\citep{Kallus2022, wu2024quantifying}. 

\begin{figure}[H]
    \centering
    \includegraphics[width=0.8\textwidth]{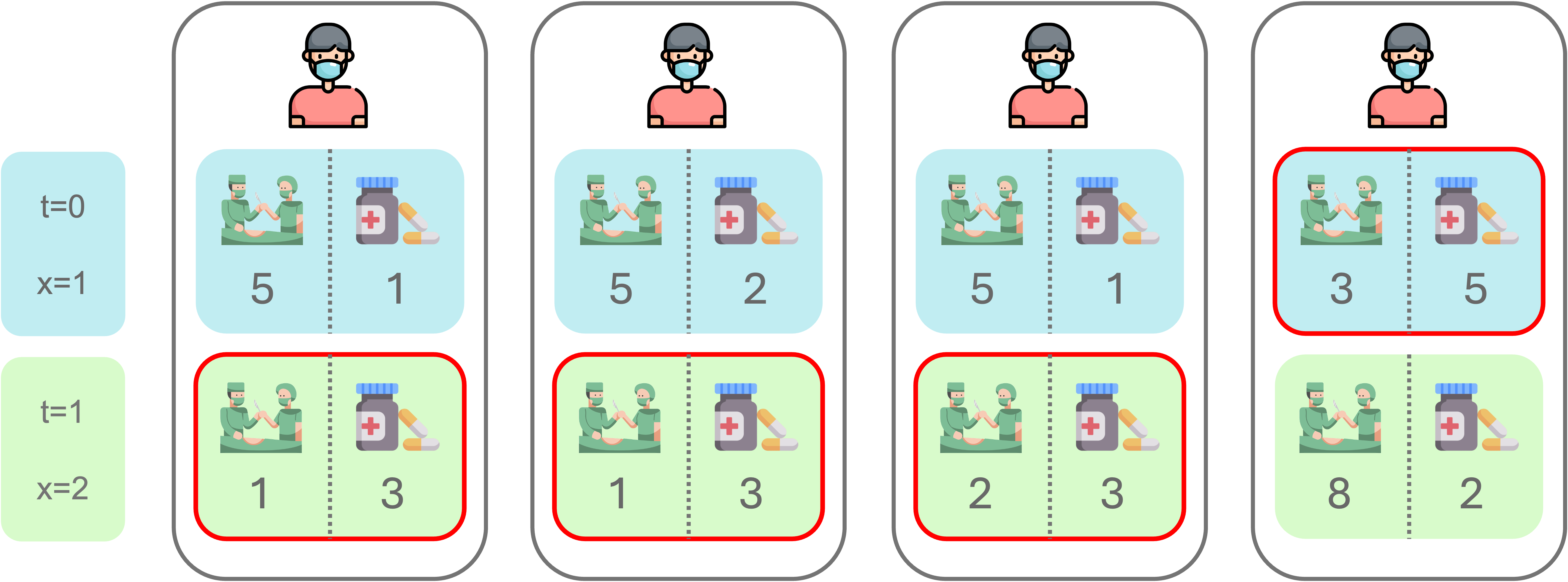}
    \caption{Illustration of potential outcomes of four patients with two actions (\textbf{surgery} or \textbf{medication}) under two time steps. Red borders highlight the cases  in which the aggressive treatment (surgery) is harmful.}
    \label{fig:intro}
\end{figure}

From Figure~\ref{fig:intro}, the optimal policy selects the action with the highest sum of outcomes at each time step, which is to apply surgery at both time steps, yielding a maximum total outcome of $18 + 12 = 30$.
This optimality, however, comes at the cost of an average harm rate of $1/2$. The harm occurs because the action of surgery results in a strictly worse outcome for some individuals (highlighted in red border) compared with the  medication. In contrast, if we adjust the policy to assign surgery only when $x=1$ (i.e., $a = \text{surgery}$ if $x=1$, and $a = \text{medication}$ if $x = 2$), the accumulated outcome slightly decreases to $18+11 = 29$. However, the average harm rate drops to only $1/4$, greatly decreasing the harm rate and improving overall safety.

The example reveals that a policy that maximizes cumulative outcomes may be suboptimal or even harmful for a substantial subset of individuals. 
This motivates our research question: how to learn a policy that balances the objective of maximizing cumulative outcomes with the constraint of minimizing individual harm, thereby mitigating safety concerns.  
This problem is particularly critical in high-stakes application scenarios. For instance, in clinical medicine, the Hippocratic Oath guides practitioners to uphold the principle of ``do no harm''~\citep{Jonsen1978, morath2005no, Wiens-etal2019}. 

This conflict between aggregate outcome and individual safety is not merely theoretical \citep{gottesman2019guidelines}. It represents a critical challenge in precision medicine. Our motivating example also arises from HIV treatment optimization, where the primary goal of maximizing overall viral suppression must be carefully weighed against the risks of suboptimal treatment responses for individual patients, such as ineffective viral control or failure to restore immune function \citep{shortreed2011informing}. We investigate this dilemma using the HIV dataset that captures the complex trade-off between maximizing average therapeutic efficacy and preventing individual treatment failures. 

\subsection{Contribution}
This article makes the following contributions. 

We first demonstrate that maximizing expected outcomes 
does not guarantee individual safety. 
To address this issue, we formalize the problem of individual harm from a counterfactual perspective. Without loss of generality, for any ordered actions $a \succ a'$, 
we say an individual suffers harm when the outcome under action $a$ is strictly worse than the potential outcome under action $a'$. We introduce two specific metrics, Harm Rate, which measures the frequency of harm, and Harm Value, which measures the magnitude of outcome loss. 
The order of actions is specified a priori according to the context of the problem, such as the severity of an intervention or the cost associated with an action. This setting is of particular interest because it naturally raises concerns about harm: higher-order actions do not necessarily produce better outcomes for all individuals and, in some cases, may even result in harm to many individuals, as revealed in Section \ref{sec-motivation}.  

To learn a policy that accounts for harm and ensures safety, we develop a two-stage procedure to learn a safety-aware policy by incorporating harm information into policy optimization.  
In the first stage, we estimate counterfactual harm for each data point and adjust the original outcomes by a harm penalty. The adjusted outcomes represent the actual utility after accounting for the risk of individual safety.
In the second stage, we apply the Fitted Q-Iteration (FQI) algorithm on this modified dataset to learn a safety-aware policy. 
Then, we establish finite-sample theoretical guarantees for the performance of the proposed algorithm.   
We show that the sub-optimality gap of the learned policy is upper-bounded by three components: the estimation error of the counterfactual harm (Stage 1), the estimation error arising from FQI (Stage 2), and an initialization bias. We also upper bound the harm rate, demonstrating that it remains well-controlled under a margin condition.

We demonstrate that our method effectively reduces harm while maintaining relatively high outcomes in simulations.  
We further apply our procedure to a real-world HIV dataset, demonstrating its potential in addressing the trade-off between efficacy and toxicity.
These numerical results indicate that our approach can significantly enhance safety in practice.

Finally, we remark that the proposed algorithm is general and can be seamlessly integrated into existing RL algorithms. Researchers and practitioners can apply our harm-aware adjustment to any learned policy to enhance safety without modifying the underlying learning algorithm. This versatility supports broad applicability in safety-critical domains such as healthcare, finance, and autonomous systems, where individual-level reliability and safety are paramount.

\subsection{Related Work} 
In the statistics literature, RL is closely related to a huge literature on learning dynamic treatment regimes that maximize the expected cumulative outcome of patients within a population~\citep[see e.g.,][and the references therein]{chakraborty2013statistical}. 
These works are later extended to  
Markov Decision Processes (MDP) settings~\citep{ertefaie2018constructing,luckett2020estimating, liao2022batch, shi2024statistically, li2024settling, zhou2024estimating}. 
These algorithms form the foundation for many applications in healthcare, economics, and operations research, where the primary objective is to improve average performance across individuals or environments.

Beyond outcome maximization, another line of works considers additional factors such as fairness and risk sensitivity. Fairness-aware policy learning seeks to mitigate disparate treatment across demographic or contextual groups \citep{viviano2024fair, liu2025fairness}. 
\citet{kusner2017counterfactual} 
introduce the concept of counterfactual fairness, ensuring that predictions remain invariant to sensitive attributes in causal models. Later studies extend this principle to RL, incorporating fairness constraints in reward or transition dynamics~\citep{joseph2016fairness, zhang2018mitigating, chen2024learning, wang2025counterfactually}. 
Another important direction is risk-sensitive and constrained RL, where the goal is to control not only the mean return but also higher-order moments or tail risks. Variance-constrained, quantile- and CVaR-based  (Conditional Value-at-Risk)~algorithms
explicitly trade off expected reward and variability, yielding safer and more predictable policies \citep{chow2018lyapunov, wang2018quantile, tang2020worst, mo2021learning, yang2023riskaverse, fang2023fairness, vaskov2024no, zhong2025risk}. 
Distributionally robust RL further addresses uncertainty in transition or reward models by optimizing worst-case performance under plausible model perturbations~\citep{pinto2017robust, derman2018soft, wang2020robust}. 
These approaches ensure that policies respect safety or resource constraints while pursuing near-optimal rewards. Collectively, these studies move beyond the pure reward-maximization paradigm by emphasizing fairness, robustness, or stability.

Our work falls into the second line of research. We extend the existing frameworks by considering the individual-level harm in the policy learning. Our approach aligns with the recent movement toward responsible and interpretable policy learning~\citep{thomas2019preventing, jiang2022constrained, shi2023risk}, and offers a principled way to trade off performance and safety in real-world decision-making systems.

\emph{Paper Organization.} 
In Section \ref{sec:prob}, we present the problem setup, define harm under the counterfactual framework, provide motivating examples, and introduce the identifiability assumptions. Section \ref{sec:policy} outlines our two-stage algorithm and establishes its finite-sample theoretical guarantees, including performance gap bounds and harm rate control. 
Section \ref{sec:num} presents synthetic experiments and a real-world application, demonstrating the effectiveness and safety awareness of the proposed method. Section \ref{sec-conclusion} concludes with a discussion.


\section{Statistical Framework}\label{sec:prob}

We consider a sequential decision making problem, formulated as an MDP $\mathcal{M} = \{\mathcal{X}, \mathcal{A}, \mathcal{P}, r, \gamma\}$. 
We assume $N$ independent individuals, each following the same decision process. For individual $i$, at each time point $t$, the current state is represented by covariates $X_{i,t} \in \mathcal{X}$. An action $A_{i,t} \in \mathcal{A}$ is then selected, after which an immediate random scalar outcome $Y_{i,t}$ is observed. The state subsequently changes to $X_{i,t+1}$ according to a transition function $\mathcal{P}(X_{i,t+1} \mid X_{i,t}, A_{i,t})$. The observed data consists of $N$ sample trajectories, 
$\{ (X_{i,t}, A_{i,t}, Y_{i,t}): t = 0, ..., T_i, i = 1, ..., N \}$, where $T_i$ is the length of horizon for individual $i$. For simplicity, we let $T_i = T$ for all $i$. Without loss of generality, we assume
that a larger value of $Y_i$ is desirable, and 
we drop the subscript $i$ for a generic individual. For clarity, we summarize the above data assumptions below. 
\begin{assumption}[Data Generating Process]  \label{assump1} Suppose that 
\par 
(a) (Markovianity) For any $t$, $X_{t} \indep \{(X_s, A_s, Y_s): 0 \leq s \leq  t -1\} \mid (X_{t-1}, A_{t-1})$. 

(b) (Conditional Mean Independence) $\E[ Y_t \mid X_t =x, A_t = a, \{(X_s, A_s, Y_s): 0 \leq s < t\} ] =\E[ Y_t \mid X_t =x, A_t = a]=r(x,a)$.

(c) (Stationary) For any $t$, $\Pr(X_{t+1} = x'\mid A_t = a, X_t = x) = \Pr(X_1 = x' \mid X_0 = x, A_0 = a)$. 
\end{assumption}

The above conditions are standard and widely imposed in RL literature~\citep{sutton1998reinforcement,kallus2022efficiently,shi2024value}. Assumptions~\ref{assump1}(a) and \ref{assump1}(b) imply that the future state and the conditional mean of the immediate outcome are independent of the past observations given the current state-action pair. Assumption~\ref{assump1}(c) further requires that the state transition and outcome regression function are invariant over time. 

At each time step $t$, we define the history 
as $H_t = (X_0, A_0, Y_0, \dots, X_{t-1}, A_{t-1}, Y_{t-1}, X_t),$ which includes all past covariates, actions, outcomes, and current state up to that point.  
A policy is a sequence of decision rules $\pi = \{\pi_t\}_{t\ge 0}$, where each $\pi_t$ maps the space of $H_t$ to a probability mass function on the action space, which is the shorthand of $\pi_t(\cdot \mid H_t)$. 
Under a given policy $\pi$, at time $t$, an agent chooses action $A_{t}$ at state $X_t$ with probability $\pi_t(A_t\mid H_t)$ and receives an immediate outcome $Y_t$. The discounted cumulative outcome is defined as $\sum_{t=0}^{+\infty} \gamma^t Y_{t}$, where the discount factor $\gamma \in [0,1)$ represents the trade-off between the immediate and future outcomes.

Under Assumption~\ref{assump1}, there exists an optimal stationary policy whose action selection depends on the history only through the current state, with this dependence being time-invariant, that its expected cumulative outcome is no smaller than that of any history-dependent policy \citep{puterman2014markov}.

\subsection{Counterfactual Harm}


Similar to~\citet{Kallus2022} and~\citet{wu2024quantifying}, we define the harm using the potential outcome framework. 
For each time step $t$ and action $a \in \mathcal{A}$, let $Y_t(a)$ denote the potential outcome that would be observed if action $a$ were taken at time $t$. By consistency~\citep{Hernan-Robins2020},  the observed outcome $Y_t$ is equal to $Y_t(A_t)$.  


\begin{definition}[Action-Relative Harm] \label{def-harm-rate}
For any two actions $a,a'$ such that action $a$ performs better than $a'$ on average, an individual suffers harm at time step $t$ if the outcome under $a$ is lower than under $a'$, i.e., $Y_t(a) - Y_t(a')<0$, with the value of such harm being $Y_t(a') - Y_t(a)$.
Correspondingly, given the current state $X_t = x$ (or for the subpopulation of $X_t = x$), the \textbf{harm rate} 
is defined as 
\begin{equation*} \label{eq:HR}
    \HR(x;a,a')=\Pr\left(Y_t(a) - Y_t(a') <0 \mid X_t=x \right),
\end{equation*}
where the probability is taken with respect to the population of individuals. 
Similarly, the \textbf{harm value} at time step $t$ and state $X_t = x$ is defined as 
\begin{equation*}
    \HQ(x; a,a') 
    =  \E \Big[ (Y_t(a') - Y_t(a)) \I(Y_t(a) - Y_t(a')<0)\mid X_t=x \Big].
\end{equation*} 
\end{definition}

\begin{remark}
Different from~\citet{Kallus2022} and~\citet{wu2024quantifying}, we generalize the concept of harm by extending the binary treatment setting to an ordered multi-valued treatment framework. We also incorporate the value of harm into consideration, providing a more comprehensive characterization of individual-level effect.


    
\end{remark}


In the subpopulation where $X_t = x$, the harm rate $\HR(x; a, a')$  quantifies the proportion of individuals harmed by choosing action $a$ relative to $a'$, and the harm value $\HQ(x;a,a')$ captures the expected loss in outcome at time $t$ resulting from selecting action $a$ rather than $a'$. 
Notably, even if the conditional mean outcome under action $a$ is larger than that under action $a'$, a high variance in the difference $Y_t(a) - Y_t(a')$ can still result in a large harm rate (or value). 

Section \ref{sec-motivation} illustrates that accounting for harm is crucial to avoid overly aggressive policies and enhancing decision safety. Hereafter, we focus on harm rate, and the associated results for harm value are relegated to the Supplementary Material. 


\subsection{Problem Definition}  \label{sec2-2}

In this article, our goal is to learn a safe policy that achieves high expected outcomes with a low harm rate.  
To achieve this, we define the utility at time $t$ under action $a$ as 
\begin{equation} \label{eq:R-hr}
    R_t(a) = Y_t(a) - \beta \cdot \HR(X_t; a, a')
\end{equation} 
where \( \beta > 0 \) is a penalty parameter that controls the trade-off between outcome and harm, and $a'$ is the reference action. We may interpret $\beta$ as the risk-aversion factor, and the utility is the outcome minus the risk-aversion factor multiplied by the harm rate. 
The reference action could be pre-specified or determined by a baseline policy $\pi^0$ (see Section \ref{sec3-3} for details).
The formulation in \eqref{eq:R-hr} reflects a core principle in safe decision making: an action with a high expected outcome may be undesirable if it carries substantial risk of harm. By incorporating the harm rate into the outcome function, we encourage the policy to favor not only effectiveness but also safety.

We define return as the discounted cumulative utility, denoted as $\sum_{t=0}^{\infty} \gamma^t R_{t}$. Under policy $\pi$, we define the value function $V^\pi(x)$ to be the expected return, starting from state $x$ as $V^\pi(x) = \E^\pi\Big[\sum_{t=0}^{\infty} \gamma^t R_{t} \mid X_0 = x\Big]$, and action-value function $Q^\pi(x,a)$ to be the expected return starting from state $x$ and action $a$ as $Q^\pi(x,a) = \E^\pi\Big[\sum_{t=0}^{\infty} \gamma^t R_{t} \mid X_0 = x, A_0 = a\Big]$, where $\E^\pi$ denotes the expectation assuming that the system follows policy $\pi$.
Let $J(\pi) = \E [ V^\pi(X_0)]$ denote the expectation aggregated over the initial state. Our goal is to find the target policy $\pi^*$ that selects the action that maximizes $J(\pi)$. Let 
\begin{equation}  \label{eq7}
    J^* = J(\pi^*) = 
    \max_{\pi} J(\pi).
\end{equation}
Quantifying harm is crucial in many real-world applications. We present two illustrative examples to highlight  
the practical significance of incorporating counterfactual harm into decision-making.

\begin{example}[Healthcare]  
Medical treatment for cancer may include surgery, chemotherapy, and conservative medication, listed in order of decreasing clinical aggressiveness. For instance, early-stage patients may experience severe side effects or complications from surgery that diminish quality of life compared with less invasive options. Thus, choosing a more aggressive intervention does not always improve, and may even worsen, patient outcomes. Indeed, \citet{wilt2017follow} showed that although surgery is clinically more aggressive, it failed to improve survival outcomes while adversely affecting quality of life for many patients. This illustrates that interventions can cause harm when applied too broadly, underscoring the importance of accounting for individual heterogeneity in treatment response. 
\end{example}

\begin{example}[Education]
In education policy, disciplinary measures intended to correct behavior can be counterproductive, leading to worse educational outcomes for certain individuals. Consider a school administrator choosing among expulsion, detention, or tutoring for a student exhibiting behavioral issues. While the goal is to improve the student's academic trajectory, severe disciplinary actions can sometimes actively induce failure on the very metric they aim to improve, as highlighted by studies on the negative impacts of exclusionary sanctions~\citep{morris2016punishment,mowen2016school}.




\end{example}

\subsection{Identifiability Assumptions}\label{sec:identi}
A key challenge in learning safe policies is that we observe the realized outcome $Y_t(a)$ only for the action actually taken, i.e., when $A_t = a$. However, the utility function $R_t(a)$ depends on the harm rate $\HR(X_t; a, a')$, which requires identifying the distribution of potential outcomes for each individual under different actions~\citep{Wu-Mao2025}. To address this, we introduce the following two identifiability assumptions.

\begin{assumption}[No Unmeasured Confounders]\label{ass:ign}
For each $t < T$, $X_t$ blocks all backdoor paths from $A_t$ to $X_{t+1}$ and from $A_t$ to $Y_t$.
\end{assumption}

Assumption \ref{ass:ign} is standard in causal inference~\citep{pearl2009causality} and RL. It implies that there are no unmeasured confounders between $A_t$ and $X_{t+1}$, as well as $A_t$ and $Y_t$. It is crucial to identify the treatment effect of $A_t$ on $Y_t$ at each time step $t$.
However, this assumption alone is insufficient to identify the harm rate, since its definition depends on the joint distribution of potential outcomes. 
To proceed, we adopt a copula-based approach to model the association between potential outcomes conditional on covariates~\citep{Lu-etal2025, Zhang-Yang2025}.

\begin{assumption}[Copula Models] \label{assump-copula}
    Let $f(Y_t(a)=\cdot, Y_t(a') = \cdot \mid X_t = x)$ be the joint density of $(Y_t(a), Y_t(a'))$ given $X=x$, 
    $f(Y_t(a) = \cdot \mid X_t = x)$ be the marginal density of $Y_t(a)$ given $X=x$,  
     $F(Y_t(a) = \cdot \mid X_t  =x)$ be the cumulative distribution function of $Y_t(a)$ given $X=x$, $c_{\rho}(\cdot, \cdot)$ be the specified copula function, and $\rho$ be the pre-specified correlation between $Y_t(a)$ and $Y_t(a')$ given $X$. Assume $f(Y_t(a)=y_a, Y_t(a') = y_{a'} \mid X_t = x) = c_{\rho}\big\{ F(Y_t(a) = y_a\mid X=x), F(Y_t(a') = y_{a'}\mid X=x) \big\} f(Y_t(a) = y_a \mid X=x) f(Y_t(a') = y_{a'} \mid X=x)$. 
\end{assumption}

 Copula models are classical tools for linking marginal and joint distributions~\citep{Jaworski-etal-compula2010, Bartolucci01062011, Sun-etal2024}. Given a known copula function and association parameter $\rho$, the joint distribution of $(Y_t(a), Y_t(a'))$ given covariates is identifiable from the marginal distributions of $Y_t(a)$ and $Y_t(a')$ given covariates. 
 For example, for the commonly used Gaussian copula with parameter $\rho$, under Assumptions \ref{assump1}–\ref{assump-copula}, $\HR(x; a, a')$ is identified as 
\begin{equation}\label{eqn:copulaharm}
    \Phi \left \{  (r(x, a') - r(x,a)) / \sigma(x; a, a') \right \},
\end{equation} 
where $\Phi(\cdot)$ is the cumulative distribution function  of a standard normal distribution, $r(x, a) = \E[ Y_t \mid X_t = x, A_t=a]$, and 
$\sigma(x; a, a') = \sqrt{\sigma^2(x,a) + \sigma^2(x,a')  - 2 \rho \sigma(x,a) \sigma(x,a')}$ with $\sigma^2(x,a) = \text{Var}(Y_t\mid X_t = x, A_t = a)$, see the Supplementary Material for details. Although the association parameter $\rho$ in copula models is not identifiable from the observed data, we recommend treating it as a sensitivity parameter in real-world applications. 
Varying $\rho$ across a plausible range, rather than holding it fixed, provides a way to evaluate the robustness of the resulting estimates and conclusions.
\section{Counterfactually Safe Policy Learning}\label{sec:policy}

In this section, we first introduce the proposed safe policy learning method, present its finite-sample properties, and then establish its connection to a constrained optimization problem. Finally, we extend our methodology to settings where the reference action is specified by a reference policy. 

\subsection{Safety-Aware Policy Learning}

The proposed counterfactually safe policy learning method involves two steps: (1) data preprocessing to obtain the observed pseudo-utility for each trajectory and (2) a Q-learning algorithm to learn the safety-aware policy.

The first step is summarized in Algorithm \ref{algo:est}. It leverages the observed data to estimate the utility for each individual at each time point, and thus incorporates safety considerations. This procedure transforms the raw outcomes into a safety-aware representation that explicitly characterizes counterfactual harm. 
After obtaining the transformed dataset $\widehat{\D}$ through Algorithm \ref{algo:est}, we proceed to policy learning, summarized in Algorithm \ref{algo:FQI}.

\begin{algorithm}[t]
\caption{Utility Estimation}
\label{algo:est}
\begin{algorithmic}[1]
\State \textbf{Input:} Dataset $\mathcal{D} = \{(X_{it}, A_{it}, Y_{it})_{i=1}^N, t = 0, \ldots, T\}$; a given reference action $a'$. 
\State Estimate $r(x,a)$, $\sigma(x, a)$ and 
      plug-in these estimators into \eqref{eqn:copulaharm} to construct the resulting harm-rate estimator 
      $\widehat{\HR}(x; a, a')$.


    \State Calculate the pseudo-utility $\widehat{R}_{it}$ as follows: 
    \begin{align*}
        \widehat{R}_{it} = \widehat{R}_t(A_{it}) &= Y_t(A_{it}) - \beta \cdot \widehat{\HR}(X_{it}; A_{it}, a') \label{eq:estimatedR-HR}. 
    \end{align*}
\State \textbf{Return:} Transformed dataset $\widehat{\D} = \{ (X_{it}, A_{it}, \widehat{R}_{it})_{i=1}^N, \, t = 0, \dots, T\}$. 
\end{algorithmic}
\end{algorithm}

\begin{algorithm}[t]
\caption{Fitted Q-Iteration}
\label{algo:FQI}
\begin{algorithmic}[1]
\State \textbf{Input:} Dataset $\tilde \D = \{(X_i, A_i, \widehat{R}_i, X_i')\}_{i=1}^n$, to be divided into $K$ batches $\tilde \D_1,\cdots,\tilde \D_K$, function class $\Q$ for action-value function, discount factor $\gamma$. 

\State \textbf{Initialize:} $Q_{0} \in \Q$ arbitrarily (e.g., $Q_{0}(x,a) = 0$). 
\For{$k = 1$ to $K$}
    \State Compute targets:
    $R_i^{(k)} \gets \widehat{R}_i + \gamma \cdot \max_{a' \in \mathcal{A}} Q_{k-1}(X'_i, a'),
    \quad  i \in \tilde \D_k.$
    \State Fit a new Q-function:
    $\widehat Q_{k} \gets \arg\min_{Q \in \Q} \sum_{i\in \tilde \D_k} \left( Q(X_i, A_i) - R_i^{(k)} \right)^2.$
\EndFor
\State \textbf{Return:} $\widehat Q_{K}$ and estimated optimal policy $\pi_{\widehat Q_{K}}(x) \in \arg\max_{a} \widehat Q_{K}(x,a)$.
\end{algorithmic}
\end{algorithm}

In Algorithm \ref{algo:FQI}, we aggregate the data across time as the input. Specifically, for each $t$, we take the triple $\{X_{it}, A_{it}, \widehat{R}_{it}\}$ obtained from Algorithm~\ref{algo:est} together with the next state $X_{i,t+1}$ (denoted as $X_i'$ in Algorithm \ref{algo:FQI} for ease of presentation), and flatten them into a pooled dataset $\tilde \D=\{(X_i, A_i, \widehat{R}_i, X'_i)\}_{i=1}^n$, where $n=NT$. We then use such quaternions to search for optimal stationary policy over all $n$ samples jointly. 
Algorithm \ref{algo:FQI} implements the Fitted Q-Iteration (FQI) framework, which leverages the pooled dataset to approximate the optimal Q-function through a sequence of supervised learning tasks. In each iteration, the algorithm computes the regression targets using the Bellman optimality operator based on the previous estimate $\widehat Q_{k-1}$, and updates the current state-action value function $\widehat Q_{k}$. After a pre-specified number of iterations $K$, we obtain the estimated Q-function which induces a stationary policy that we desire. 

Next, we present the finite-sample analysis of the proposed method, focusing on the regret, defined as the difference between the expected return under the optimal policy and that under the estimated policy, i.e. $J(\pi^*) - J(\pi_{\widehat{Q}_K})$. Our theoretical analysis relies on several regularity conditions summarized in Assumption \ref{ass:linear} below. 
\begin{assumption}[Regularity Conditions] \label{ass:linear}
Suppose the following conditions hold:

(i) Bounded Linear Class. We assume the action-value function class is of linear of the feature map
\begin{equation*}
    \Q = \{w^\top \phi(x,a): w\in \mathbb{R}^d, \|w\|_{\infty}\leq W\},
\end{equation*}

(ii) Completeness. 
For any $Q\in\Q$, $\T Q \in \Q$, where $\T$ is the Bellman operator. 

(iii) Feature Coverage. 
There exists a constant $\lambda_0>0$ such that
\begin{equation*}
    \lambda_{\min}( \E_{x,a\sim \mu_b}[\phi(x,a)\phi(x,a)^\top] ) \geq \lambda_0, 
\end{equation*}
where $\mu_b$ is the behavior distribution.

(iv) Bounded Outcomes. 
    There exists a constant $M$ such that $|Y_t| \leq M$.
    
\end{assumption}



Assumption \ref{ass:linear}(i) indicates that given a feature map $\phi$, we restrict our function class $\Q$ that is linear in these features with bounded coefficients~\citep{hu2025fast}. Assumption \ref{ass:linear}(ii) states that the function class $\Q$ is closed under the Bellman operator $\T$. Assumption \ref{ass:linear}(iii) is commonly required to guarantee the convergence of ordinary least squares estimators~\citep{wang2020statistical}. Assumption \ref{ass:linear}(iv) posits boundedness of the outcomes. 
These assumptions are mild and standard, widely adopted in the reinforcement learning~\citep{munos2005error,antos2008learning,uehara2022review}. In practice, linear function approximation with bounded parameters is sufficiently flexible to capture a wide range of value functions while maintaining computational efficiency~\citep{agarwal2019reinforcement}. Meanwhile, the closure property under the Bellman operator ensures the stability of iterative updates, which is critical for establishing theoretical convergence results~\citep{panaganti2022robust}.



\begin{assumption}[Uniform Estimation Rate] \label{ass:est_harm}
    There exists a constant $C_0 > 0$ such that the estimator $\widehat{\HR}(x; a, a')$ uniformly approximates the true $\textup{HR}(x; a, a')$ over all $(x, a)$ with error at most $C_0 \epsilon_{n}$, i.e., $\sup_{x, a} |   \widehat{\HR}(x;  a, a')  - \HR(x; a, a')  | \leq C_0 \epsilon_{n}$.
\end{assumption}

Assumption \ref{ass:est_harm} is a high-level condition regarding the convergence rates of the estimators for the harm rate. For $x \in \mathbb{R}^p$, suppose the $\HR(x; a, a')$ belong to the $(s, L, \mathbb{R}^p)$-H\"{o}lder class. Then the convergence rate is attained by balancing the $O(h^s)$ smoothing bias of a local polynomial estimator, and the $O(\sqrt{\log N / (N h^p)})$ uniform stochastic error \citep{Fan1996, Gyorfi2002}. Specifically, $h$ denotes the bandwidth, and $h^s$ term represents the approximation error of an $s$-order Taylor expansion. The $\log N$ factor arises from the concentration inequalities required to ensure uniform consistency over the covariate support. By selecting an optimal bandwidth $h \asymp (\log N / N)^{1/(2s+p)}$, we achieve a uniform convergence rate of $O\big( (\log N / N)^{s/(2s+p)} \big)$, which is the minimax optimal rate for this function class \citep{Stone1982, Tsybakov2009}.  
We present the regret bound for the policy $\pi_{\widehat{Q}_K}$ obtained after $K$ iterations of Algorithm \ref{algo:FQI}.

\begin{theorem}[Regret Bound]\label{thm:regret}

With probability at least $1-n^{-\kappa}-d\exp{(-n\lambda_0/8)}$, we have
\begin{equation*} 
    |J(\pi^*) - J(\pi_{\widehat{Q}_K})|\leq\frac{4\epsilon_{n} \sqrt{d}}{\lambda_0(1-\gamma)^2} + \frac{2CdM\kappa \log(n)}{(1-\gamma)^3\lambda_0 \sqrt{n}} + \frac{2\gamma^K (M+\beta)}{(1-\gamma)^2},
\end{equation*}
for any $\kappa > 0$ and some constant $C > 0$, where $\beta$ is the risk-aversion factor defined in \eqref{eq:R-hr}.  
\end{theorem}

Theorem \ref{thm:regret} implies that the regret bound can be decomposed into three parts: (i) The first term captures the estimation error induced from the estimation of harm-related utility $\widehat{r}$ in Algorithm \ref{algo:est}, which is proportional to $\epsilon_{n}$, the estimation error of the harm rate, and the square root of $d$. (ii) The second term arises from the one-step FQI error caused by using Algorithm \ref{algo:FQI}. Its magnitude decays at a rate of $\log(n) n^{-1/2}$, which aligns with~\citet{hu2025fast}. (iii) the initialization bias, which decays exponentially fast with the number of iterations $K$. 
In addition, the regret bound increases with the $(1- \gamma)^{-1}$ term, which can be interpreted as the `horizon' in episodic tasks. Their dependency on the reward upper bound $M$ and the smallest eigenvalue $\lambda_0$ is aligned with existing findings in the literature~\citep{chen2019information,wang2025counterfactually}.

\subsection{Safety Assurance and Harm Bound}

We define the value functions for the outcome component and the harm component as the expected discounted sums of their respective signals $J_{Y}(\pi) = \mathbb{E}^{\pi} \left[ \sum_{t=0}^{+\infty} \gamma^t Y_t \right]$ and $J_{\HR}(\pi) = \mathbb{E}^{\pi} \left[ \sum_{t=0}^{+\infty} \gamma^t \HR(X_t; A_t, a') \right]$. 
Theorem \ref{thm:regret} establishes the convergence rate of the total utility regret $J = J_Y - \beta J_{\HR}$. However, in safety-critical scenarios, it is crucial to explicitly bound the \textit{excess harm} incurred by the learned policy, i.e., $J_{\HR}(\pi_{\widehat{Q}_K}) - J_{\HR}(\pi^*)$. 

To derive a tight bound on the excess harm, we introduce the \textit{Margin Condition} (also known as the Tsybakov noise condition,~\citeauthor{tsybakov2004optimal}, \citeyear{tsybakov2004optimal}). This condition characterizes the difficulty of the decision-making problem by quantifying the probability mass of states where the optimal action is hard to distinguish from suboptimal ones. This condition is widely used in statistical learning and reinforcement learning to characterize the difficulty of identifying the optimal action \citep{audibert2007tuning, farahmand2011action}.

\begin{assumption}[Margin Condition]\label{ass:margin}
    Let $\Delta(x) = Q^*(x, \pi^*(x)) - \max_{a \neq \pi^*(x)} Q^*(x, a)$ be the action gap of the optimal Q-function. We assume there exist constants $C_{\Delta} > 0$ and $\alpha > 0$ such that for any threshold $t > 0$:
    \begin{equation}  \label{eq3}
        \Pr\left( 0 < \Delta(x) \leq t \right) \leq C_{\Delta} t^\alpha.
    \end{equation}
\end{assumption}

A larger $\alpha$ implies that the action gap is bounded away from zero for most states, making the optimal policy easier to identify (see \citet{audibert2007fast} for more interpretation). Leveraging this assumption and the total regret bound, we establish the following bound on the safety violation (or excess harm).

\begin{theorem}[Harm Bound]\label{thm:harm_margin_regret}
    Suppose Assumption \ref{ass:margin} holds. Let $\mathcal{R}_{n}$ be the upper bound of the total regret derived in Theorem \ref{thm:regret}. 
    With probability at least $1-n^{-\kappa}-d\exp{(-n\lambda_0/8)}$, the difference in expected harm between the learned policy $\pi_{\widehat{Q}_K}$ and the optimal policy $\pi^*$ satisfies that: 
    \begin{equation*} 
        | J_{\HR}(\pi^*) - J_{\HR}(\pi_{\widehat{Q}_K}) | \leq \frac{2}{1-\gamma} (C_{\Delta} + 1) \left( (1-\gamma)\mathcal{R}_{n} \right)^{\frac{\alpha}{\alpha+1}},
    \end{equation*}
 where $C_{\Delta}$ is the constant in \eqref{eq3}. 
\end{theorem}

Theorem~\ref{thm:harm_margin_regret} shows that the bound on excess harm  of the learned policy depends critically on the bound of the total regret  $\mathcal{R}_{n}$. 
Meanwhile, under a favorable margin condition (large $\alpha$), the  bound of excess harm converges at a rate approaching that of the total regret. In addition, by an argument analogous to that in Theorem~\ref{thm:harm_margin_regret}, we can derive the bound for the outcome, as shown in Corollary \ref{prop:outcome_bound}. 
\begin{corollary}[Outcome Bound]\label{prop:outcome_bound}
    Under the same conditions as Theorem \ref{thm:harm_margin_regret}, 
    then with probability at least $1-n^{-\kappa}-d\exp{(-n\lambda_0/8)}$, the difference in expected outcome between the learned policy and the optimal policy is bounded by:
    \begin{equation*}
        | J_{Y}(\pi^*) - J_{Y}(\pi_{\widehat{Q}_K}) | \leq \frac{2 M}{1-\gamma} (C_{\Delta} + 1) \left( (1-\gamma)\mathcal{R}_{n} \right)^{\frac{\alpha}{\alpha+1}}.
    \end{equation*}
\end{corollary}


\subsection{Further Discussion on the Safe Policy} \label{sec3-3}
We provide additional discussion of the proposed method by extending the definition of the harm rate and exploring an alternative form of the target policy. 

{\bf Policy-Relative Harm and Safety-Aware Policy Improvement}.  
In Definition \ref{def-harm-rate}, the harm rate $\text{HR}(x; a, a')$ was defined with respect to a fixed reference action $a'$, which remains invariant across time steps.
To allow greater flexibility, we can instead specify the reference action through a reference policy $\pi^0 = \{\pi^0_1,...,\pi^0_T\}$.
Building on this adjustment, we could develop a safety-aware policy improvement strategy that reduces the harm rate relative to the reference policy $\pi^0$. 


\begin{definition}[Policy-Relative Harm] \label{def2}
For a given reference policy $\pi^{0}$, define the  policy-relative harm rate at time $t$ with the current sate $X_t = x$ as 
\begin{equation}
    \HR(x; \pi, \pi^0)=\HR(x; \pi_t(x),\pi^0_t(x)).
\end{equation}
\end{definition}

Definition \ref{def2} naturally extends the concept of harm rate in Definition \ref{def-harm-rate} from the action level to the policy level, thereby enabling comparisons between policies. Starting from any feasible baseline policy $\pi^0$, we can incorporate this definition into the policy learning framework to explicitly account for potential harm. This approach allows us to search for improved policies that consider harm and ultimately obtain a policy that is safer than $\pi^0$.

{\bf An Alternative Formulation of the Target Policy}.  
In Section \ref{sec2-2}, we define the target policy $\pi^*$ that maximizes $J(\pi) = J_Y(\pi) -  \beta \cdot J_{\HR}(\pi)$, where  $\beta$ was interpreted as a risk-aversion factor.
An alternative formulation of the target policy can be constructed by directly specifying the final harm rate induced by the policy. Specifically, we define $\bar{\pi}^*$ as the policy that satisfies 
 \begin{align*} 
 \begin{cases}
 \max_{\pi} & J_Y(\pi) \\
  \text{s.t.} &   J_{\HR}(\pi) \le \delta.
 \end{cases}
 \end{align*} 
For a given $\beta$, we can show that the policy $\bar{\pi}^*$  coincides with $\pi^*$ for an appropriate choice of $\delta$, see the Supplementary Material for more details.  

\section{Numerical Experiments}\label{sec:num} 

In this section, we evaluate the finite-sample performance of the proposed harm-aware policy learning method through simulation studies and further demonstrate its practical effectiveness using a real-world application. Throughout this section, the competing methods are summarized as follows: 
\begin{itemize}
    \item[(1)] Harm-unaware policy ({\bf Unaware}). 
It follows the same learning procedure as the proposed approach (Algorithm \ref{algo:FQI}), except that the estimated utility $\hat r_i$ is replaced by the observed outcome $y_i$. In other words, it ignores potential harm and seeks to learn a policy that maximizes the expected discounted cumulative outcome alone.

  \item[(2)] Behavior policy ({\bf Behavior}). In the simulation study, this corresponds to the data-generating mechanism of the action $A$, which is known. In the real-world application, it is approximated by the empirical frequency of each action, computed as the number of times the action is taken divided by the total sample size. 

  \item[(3)]  Random policy ({\bf Random}). 
It selects actions uniformly at random at each time step for every state, without incorporating information about past outcomes or state. 
\end{itemize} 
Moreover, the implementation details of the proposed method are provided in the Supplementary Material.

\subsection{Simulation}

In the simulation, we consider the following two data-generating mechanisms: 

 {\bf Linear}. The state transition function is linear: $X_{i,t+1} = 0.8X_{i,t} - 0.2 + 0.3A_{i,t} + \omega_{i,t}$. The binary action $A_{i,t}\in\{0,1\}$ 
 is generated according to
 $\pi^b(A_{i,t}=1 \mid X_{i,t})= 1/(1+e^{-0.5 X_{i,t}})$. The outcome function is also linear:  $Y_{i,t} = 0.3 + 0.4 X_{i,t} - 0.6 A_{i,t} X_{i,t} + \nu_{i,t}$, where $\omega_{i,t}\sim N(0,0.1)$ and $\nu_{i,t}\sim N(0,0.05)$ represent independent random noises.  
 
{\bf Non-Linear}. The state transition function is non-linear: $X_{i,t+1} = \tanh(0.7 X_{i,t} + 0.5A_{i,t} - 0.25) + 0.25\sin(1.3X_{i,t} + 0.5A_{i,t}) + \omega_{i,t}$. The binary action $A_{i,t}\in\{0,1\}$ is generated from $\pi^b(A_{i,t}=1\mid X_{i,t})= 1/(1+e^{-0.5 X_{i,t}})$. The outcome function is also non-linear: 
$Y_{i,t} = 0.3 + 0.25\sin(X_{i,t}+0.4A_{i,t}) + 0.15 (X_{i,t}+0.3A_{i,t})^2+0.2A_{i,t}\cos(1.5X_{i,t})-0.3A_{i,t}+ \nu_{i,t}$, where $\omega_{i,t}\sim N(0,0.1)$ and $\nu_{i,t}\sim N(0,0.1)$ represent independent random noises.   

The above setups include both linear and nonlinear cases to assess the robustness of the proposed method.
Each simulation scenario is replicated 100 times, and the methods are evaluated using two metrics: 
the \emph{discounted outcome} and the \emph{average harm}. Specifically, for a general learned policy $\hat \pi$, we define the discounted outcome under $\hat \pi$ as $(NT)^{-1}\sum_{i, t} \gamma^t Y_{i,t}(\hat\pi)$ with $Y_{i,t}(\hat\pi) = \hat \pi(X_{i,t}) Y_{i,t}(1) + (1- \hat \pi(X_{i,t})) Y_{i,t}(0)$, which quantifies the overall benefit, and the average harm under $\hat \pi$ as $(NT)^{-1}\sum_{i,t} \{Y_{i,t}(0) - Y_{i,t}(1)\}  \cdot \I\{Y_{i,t}(0) > Y_{i,t}(1)\}$, which captures the mean magnitude of harm. 



\begin{figure}[htbp]
    \centering
    \includegraphics[width=0.9\textwidth]{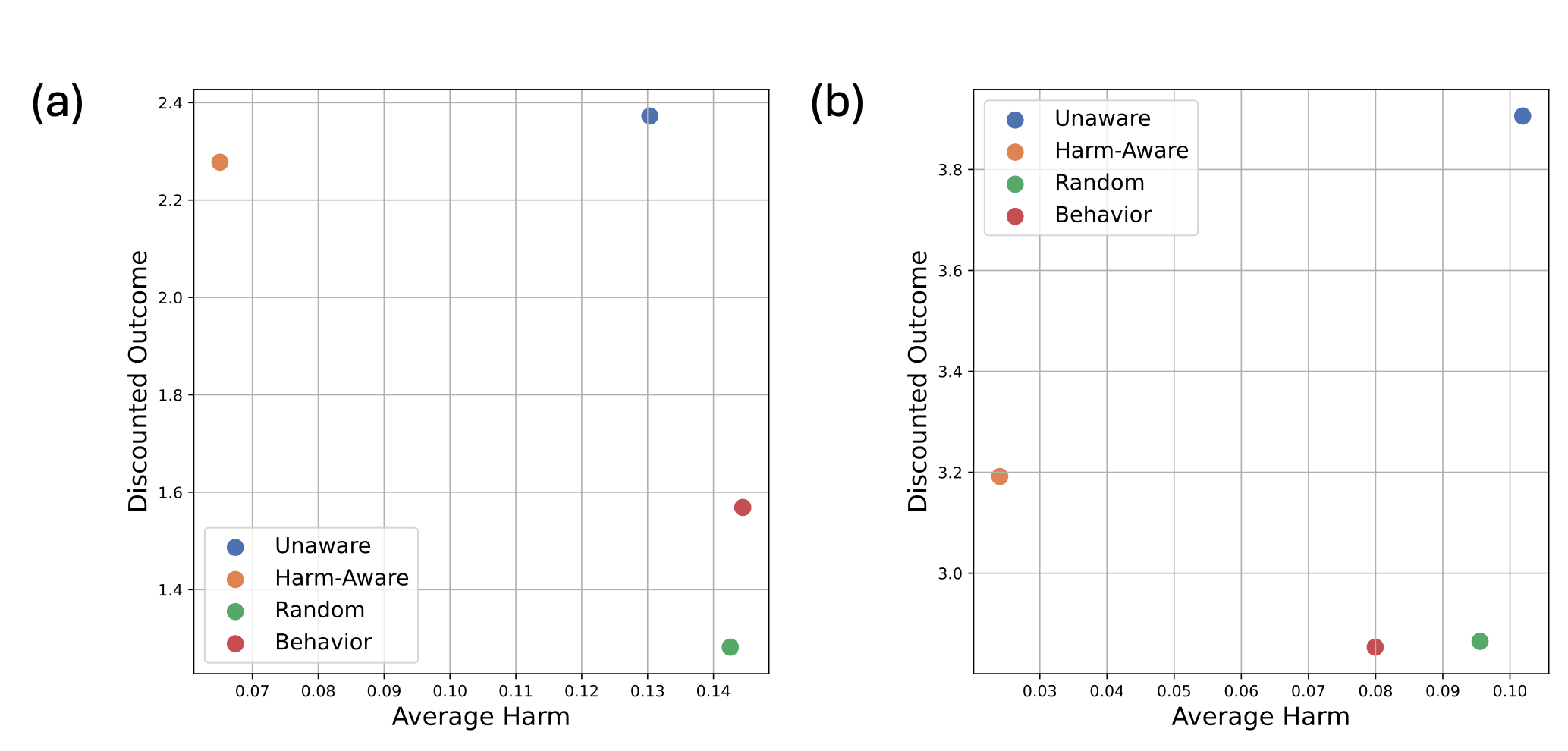}
    \caption{Comparisons of outcome versus harm under (a) linear setting and with risk-aversion factor  $\beta=0.5$ and (b) non-linear setting with risk-aversion factor $\beta=1.0$.}
    \label{fig:outcome_harm}
\end{figure}

Figure~\ref{fig:outcome_harm} presents the trade-off between discounted outcome and average harm across different policy learning methods, with a sample size of $N = 1000$ and a time horizon of $T = 20$. In both the linear and nonlinear cases, our harm-aware method achieves the lowest average harm, well below all competing approaches. In terms of discounted outcome, it attains the second-best performance, just slightly below the harm-unaware method in the linear case. These results demonstrate that our approach effectively balances outcome improvement and harm reduction, indicating its practical utility.

\begin{figure}[h]
    \centering
    \includegraphics[width=0.9\textwidth]{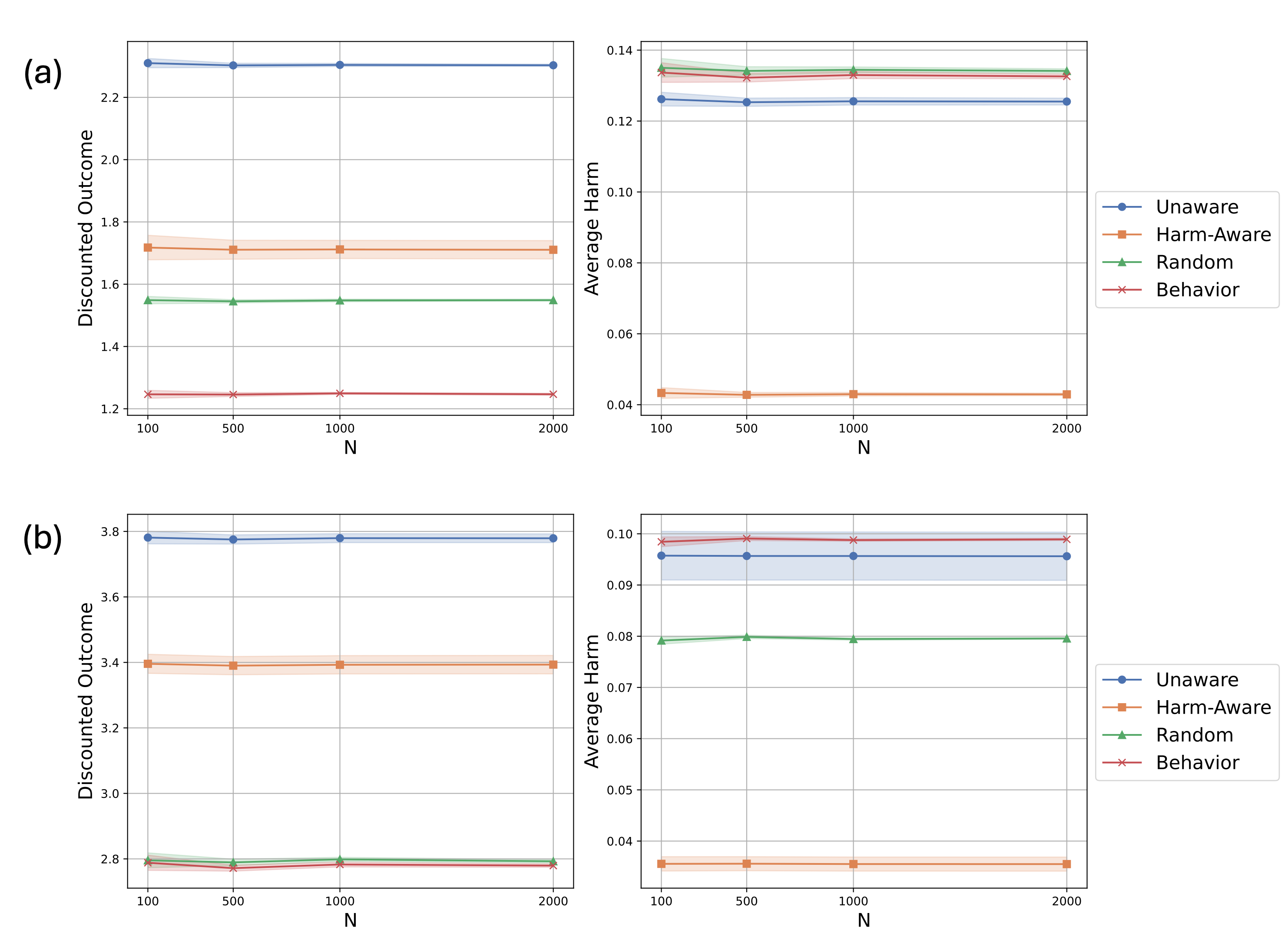}
    \caption{Comparisons of outcome and harm versus sample size $N$ under (a) linear setting with $\beta=0.5$ and (b) non-linear setting with $\beta=1.0$.}
    \label{fig:outcome_harm_N}
\end{figure}

Figure~\ref{fig:outcome_harm_N} show the results of the discounted outcome and average harm across different sample sizes.  
The proposed harm-aware policy consistently achieves the lowest average harm while maintaining the second-highest discounted outcome across all sample sizes. As $N$ increases, the shaded area representing the 95\% confidence interval 
become narrower, indicating improved policy learning stability. These results suggest that our method remains robust and reliable as the sample size grows, effectively balancing benefit and safety.


We also analyze the sensitivity of the evaluation metrics with respect to the risk-aversion factor $\beta$, which controls the trade-off between outcome and harm. Figure~\ref{fig:penalty} illustrates how the discounted outcome and average harm vary as $\beta$ changes. As $\beta$ increases, both the discounted outcome and the average harm decrease, indicating that a stronger risk-aversion penalty effectively suppresses harm, though it slightly reduces the outcome. These results provide insights into the influence of the risk-aversion factor on both benefit and safety.  Detailed numerical data for the figures are available in the Supplementary Material.

\begin{figure}[t]
    \centering
    \includegraphics[width=0.9\textwidth]{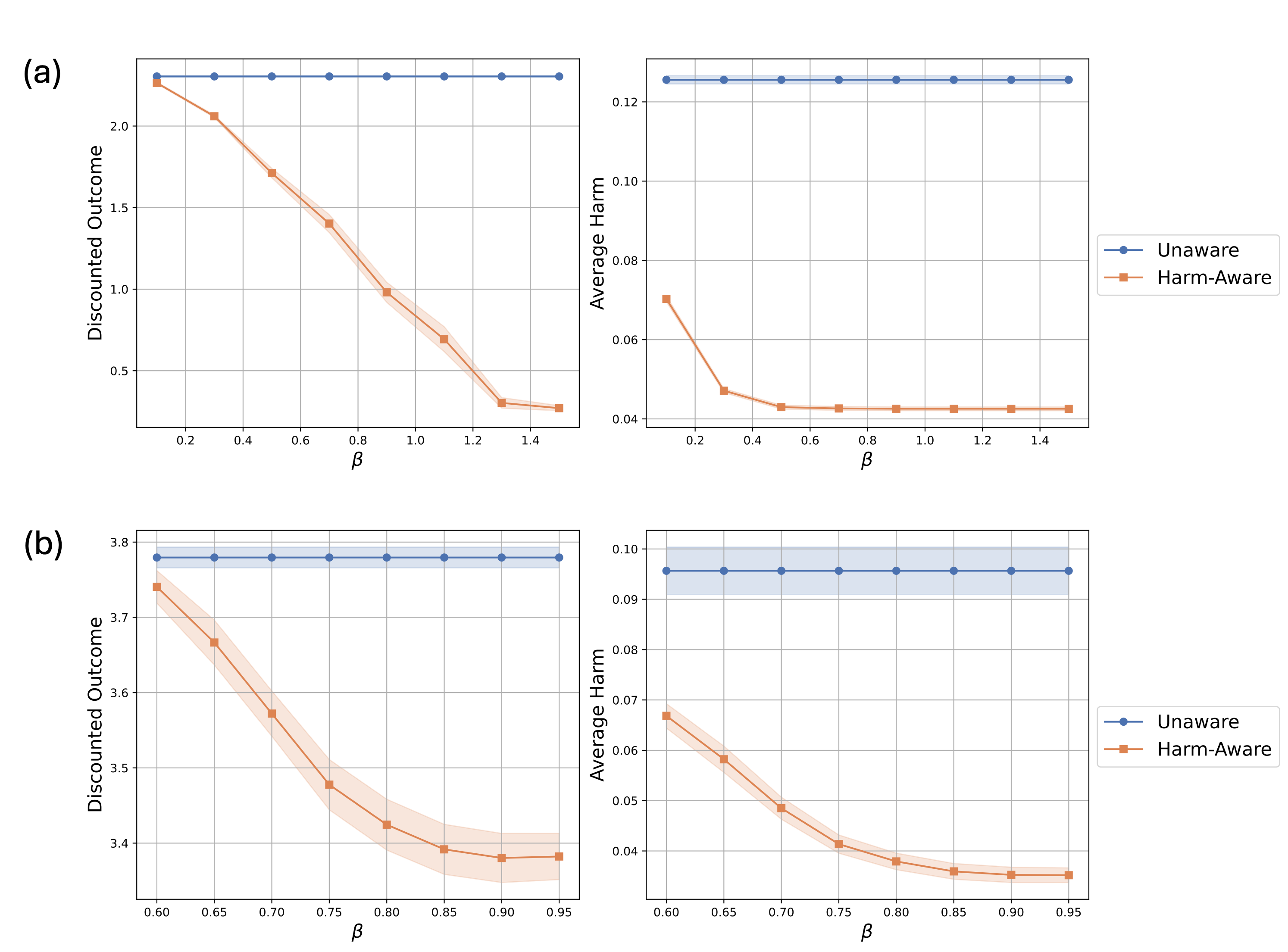}
    \caption{Comparison of outcome and harm across different values of risk-aversion factor $\beta$ under (a) the linear setting and (b) non-linear setting.}
\label{fig:penalty}
\end{figure}

\subsection{Application}

{\bf Background.}  
Human Immunodeficiency Virus (HIV) is a severe infectious disease that poses a major global health challenge. It continues to be one of the most pressing public health concerns worldwide, causing over one million deaths annually and necessitating sustained investment in both prevention and treatment efforts~\citep{gunthard2016antiretroviral}. The primary treatment strategies for HIV include the use of base drugs and non-nucleoside reverse transcriptase inhibitors (NNRTIs)~\citep{de2010non,bertagnolio2021clinical}. While most existing studies have primarily focused on assessing the clinical efficacy of different treatment regimens, they often do not account for the potential harms associated with these strategies.
In contrast, our study takes a novel perspective to analyze the same dataset. Specifically, we aim to: (1) evaluate the outcomes and harms associated with historical treatment policies (behavior policies), and (2) investigate whether improved treatment strategies can be developed that achieve higher clinical outcomes while simultaneously reducing potential harm.

{\bf Data.} We apply our proposed method to the HIV dataset in~\citet{kuo2022health}, using the outcome function from~\citet{parbhoo2017combining}. The dataset includes 8,916 synthetic patients, each with monthly measurements over a 60-month period, providing longitudinal time-series data suitable for evaluating dynamic treatment strategies. The action space is defined using two treatment components: the base drug combinations (Base Drug Combo) and complementary non-nucleoside reverse transcriptase inhibitors (Comp. NNRTI). To simplify the analysis, we map each component into two categories. For the Base Drug Combo variable, values ${0,1,2}$ are grouped as $Base = 0$, and values ${3,4,5}$ as $Base = 1$. Similarly, for the Comp. NNRTI variable, values ${0,1}$ are grouped as $Comp = 0$, and values ${2,3}$ as $Comp = 1$.
By taking the Cartesian product of these two binary components, we construct an action space consisting of four possible treatment actions, represented as $A = 2 \cdot Base + Comp$, with $A \in \{0,1,2,3\}$.  This mapping approach allows us to systematically capture the combinations of treatment options while maintaining a manageable action space.  In our analysis, we set the reference action as $a' = 0$ in equation~\eqref{eq:HR}, while all remaining patient characteristics are treated as state features to inform policy learning.

For real-world data, we only observe outcomes corresponding to the actions that were actually taken. As a result, counterfactual outcomes under an arbitrary policy cannot be directly computed. To address this limitation, we employ weighted importance sampling (WIS) estimators to estimate the expected outcome and harm under the learned policy. Specifically, let $\hat \pi(X_i)$ denote the learned policy and $\hat \pi_b(A_i \mid X_i)$ denote the estimated behavior policy. The WIS estimators of the expected outcome and harm under the learned policy are then given by
\begin{align*}
\frac{\sum_{t=0}^T\sum_{i=1}^N \I\{A_{i,t} = \hat \pi(X_{i,t})\}  Y_{i,t} / \hat \pi_b(A_{i,t}\mid X_{i,t})}
{\sum_{t=0}^T\sum_{i=1}^N \I\{A_{i,t} = \hat\pi(X_{i,t})\} / \hat \pi_b(A_{i,t}\mid X_{i,t})}
\end{align*}
and
\begin{align*}
\frac{\sum_{t=0}^T\sum_{i=1}^N \I\{A_{i,t} = \hat \pi(X_{i,t})\} 
\widehat{\HR}_{i,t}(X_{i,t}; \hat \pi(X_{i,t}), a') / \hat \pi_b(A_{i,t} \mid X_{i,t})}
{\sum_{t=0}^T\sum_{i=1}^N \I\{A_{i,t} = \hat \pi(X_{i,t})\} / \hat \pi_b(A_{i,t}\mid X_{i,t})}.
\end{align*}
These two quantities are used to evaluate the performance of the learned policy. 

\begin{table}[t]
\caption{Estimated expected outcome and harm rate on the HIV dataset under different penalty coefficient $\beta$ and copula parameter $\rho$.}
\label{tab:penalty_rho}
\centering
\resizebox{\textwidth}{!}{
\setlength{\tabcolsep}{5pt}
\begin{tabular}{cccccccccc} 
\toprule
\multirow{2.5}{*}{\textbf{Setting}} & \multirow{2.5}{*}{\textbf{Penalty}} & \multicolumn{4}{c}{\textbf{Outcome}} & \multicolumn{4}{c}{\textbf{Harm}} \\
\cmidrule(lr){3-6} \cmidrule(lr){7-10}
 & & Unaware & Harm-aware & Random & Behavior & Unaware & Harm-aware & Random & Behavior \\
\midrule
\multirow{3}{*}{$\boldsymbol{\rho=0.0}$} 
 & 3.0 & 6.063 & 5.324 & 3.360 & 3.345 & 0.268 & \textbf{0.122*} & 0.339 & 0.383 \\
 & 5.0 & 6.063 & 5.112 & 3.360 & 3.345 & 0.268 & \textbf{0.081*} & 0.339 & 0.383 \\
 & 7.0 & 6.063 & 5.095 & 3.360 & 3.345 & 0.268 & \textbf{0.065*} & 0.339 & 0.383 \\
\midrule
\multirow{3}{*}{$\boldsymbol{\rho=0.5}$} 
 & 3.0 & 6.063 & 5.317 & 3.360 & 3.345 & 0.268 & \textbf{0.126*} & 0.339 & 0.383 \\
 & 5.0 & 6.063 & 5.085 & 3.360 & 3.345 & 0.268 & \textbf{0.109*} & 0.339 & 0.383 \\
 & 7.0 & 6.063 & 4.901 & 3.360 & 3.345 & 0.268 & \textbf{0.065*} & 0.339 & 0.383 \\
\midrule
\multirow{3}{*}{$\boldsymbol{\rho=1.0}$} 
 & 3.0 & 6.063 & 5.701 & 3.360 & 3.345 & 0.268 & \textbf{0.130*} & 0.339 & 0.383 \\
 & 5.0 & 6.063 & 5.246 & 3.360 & 3.345 & 0.268 & \textbf{0.113*} & 0.339 & 0.383 \\
 & 7.0 & 6.063 & 5.066 & 3.360 & 3.345 & 0.268 & \textbf{0.077*} & 0.339 & 0.383 \\
\bottomrule
\end{tabular}
}
\end{table}

{\bf Results.} 
We present the numerical results for various policy learning methods, evaluated under different values of the correlation parameter $\rho$ and the risk-aversion factor $\beta$ as part of a sensitivity analysis. The corresponding results are summarized in Table \ref{tab:penalty_rho}. From these results, it is evident that the harm-unaware policy achieves the highest discounted outcome but incurs a substantially higher harm rate, indicating that focusing solely on efficacy can entail considerable safety risks. In contrast, our harm-aware method substantially reduces the harm by approximately 70\% compared to the harm-unaware policy when the risk-aversion factor $\beta$ exceeds 3, 
while maintaining a comparable level of overall outcome.
Meanwhile, the historical treatment policies (Behavior) result in both lower outcomes and higher harm, suggesting that these historical strategies could be significantly improved in terms of both outcome enhancement and harm reduction.
These findings illustrate that our approach is capable of generating treatment strategies that are markedly safer without compromising primary therapeutic effectiveness, highlighting its practical value for designing risk-sensitive policies.

\section{Discussion}  \label{sec-conclusion}
In this article, we address the critical challenge of ensuring individual safety in offline reinforcement learning, where efforts to maximize aggregate returns often obscure risks to certain individuals. We propose a novel framework that formalizes counterfactual harm and develop a general two-stage algorithm for learning safety-aware policies. A key strength of our approach is its generality and modularity: the proposed harm estimation and penalty mechanisms are model-agnostic, allowing seamless integration into a wide range of existing offline RL algorithms, including FQI, CQL~\citep{kumar2020conservative,yu2021combo}, and BCQ~\citep{fujimoto2019off}, without requiring fundamental modifications to the underlying optimization procedure. This flexibility makes our method particularly well-suited for high-stakes application domains such as precision medicine, autonomous driving, and public policy, where ensuring individual-level reliability is as critical as optimizing average outcome.

There are several potential avenues for future research.
First, the identification of treatment effects in our framework depends critically on the no-unmeasured confounders assumption, which cannot be verified using only observed data. Future work could incorporate sensitivity analysis frameworks to evaluate how robust the learned policies are to potential violations of this assumption. In particular, developing methods to bound the maximum harm rate in the presence of unmeasured confounding would provide stronger safety guarantees and increase the practical reliability of harm-aware algorithms. 
Second, while our current framework focuses on a scalar outcome variable $Y$, real-world decision-making often involves multiple, potentially conflicting objectives, such as maximizing efficacy, minimizing financial cost, and ensuring psychological well-being. Extending the proposed framework to settings with multiple outcome variables, where $Y$ is represented as a random vector, would allow for a more nuanced and comprehensive definition of harm. Such an extension would require characterizing the trade-offs among different dimensions of safety and developing multi-objective harm-aware policy learning algorithms that can balance these competing objectives effectively.  

\section{Competing interests}
No competing interest is declared.





\bibliographystyle{plainnat}
\bibliography{reference}

@article{li2024settling,
    author = {Gen Li and Laixi Shi and Yuxin Chen and Yuejie Chi and Yuting Wei},
    title = {{Settling the sample complexity of model-based offline reinforcement learning}},
    volume = {52},
    journal = {The Annals of Statistics},
    number = {1},
    publisher = {Institute of Mathematical Statistics},
    pages = {233 -- 260},
    year = {2024}
}

@article{zhou2024estimating,
    author = {Wenzhuo Zhou and Ruoqing Zhu and Annie Qu},
    title = {Estimating Optimal Infinite Horizon Dynamic Treatment Regimes via pT-Learning},
    journal = {Journal of the American Statistical Association},
    volume = {119},
    number = {545},
    pages = {625--638},
    year = {2024},
    publisher = {Taylor \& Francis},
    doi = {10.1080/01621459.2022.2138760}
}

@article{shi2024value,
    author = {Chengchun Shi and Zhengling Qi and Jianing Wang and Fan Zhou},
    title = {Value Enhancement of Reinforcement Learning via Efficient and Robust Trust Region Optimization},
    journal = {Journal of the American Statistical Association},
    volume = {119},
    number = {547},
    pages = {2011--2025},
    year = {2024}
}

@article{vaskov2024no,
  title={Do no harm: A counterfactual approach to safe reinforcement learning},
  author={Vaskov, Sean and Schwarting, Wilko and Baker, Chris},
  booktitle={6th Annual Learning for Dynamics \& Control Conference},
  pages={1675--1687},
  year={2024},
  organization={PMLR}
}

@article{chen2024learning,
  title={On learning and testing of counterfactual fairness through data preprocessing},
  author={Chen, Haoyu and Lu, Wenbin and Song, Rui and Ghosh, Pulak},
  journal={Journal of the American Statistical Association},
  volume={119},
  number={546},
  pages={1286--1296},
  year={2024},
  publisher={Taylor \& Francis}
}

@article{mo2021learning,
  title={Learning optimal distributionally robust individualized treatment rules},
  author={Mo, Weibin and Qi, Zhengling and Liu, Yufeng},
  journal={Journal of the American Statistical Association},
  volume={116},
  number={534},
  pages={659--674},
  year={2021},
  publisher={Taylor \& Francis}
}

@article{Zhang-Yang2025,
	author = {Yichi Zhang and Shu Yang},
	journal = {Journal of the Royal Statistical Society: Series B (Statistical Methodology)},
	pages = {1655–-1677},
	title = {Semiparametric localized principal stratification analysis with continuous strata},
	volume = {87},
	year = {2025}}

@article{Lu-etal2025,
    author = {Lu, Sizhu and Jiang, Zhichao and Ding, Peng},
    title = {Principal stratification with continuous post-treatment variables: nonparametric identification and semiparametric estimation},
    journal = {Journal of the Royal Statistical Society Series B: Statistical Methodology},
    pages = {qkaf049},
    year = {2025}}

@book{sutton1998reinforcement,
  title={Reinforcement learning: An introduction},
  author={Sutton, Richard S and Barto, Andrew G and others},
  volume={1},
  number={1},
  year={1998},
  publisher={MIT press Cambridge}
}

@article{li2017deep,
  title={Deep reinforcement learning: An overview},
  author={Li, Yuxi},
  journal={arXiv preprint arXiv:1701.07274},
  year={2017}
}

@article{kober2013reinforcement,
  title={Reinforcement learning in robotics: A survey},
  author={Kober, Jens and Bagnell, J Andrew and Peters, Jan},
  journal={The International Journal of Robotics Research},
  volume={32},
  number={11},
  pages={1238--1274},
  year={2013},
  publisher={SAGE Publications Sage UK: London, England}
}

@article{silver2016mastering,
  title={Mastering the game of Go with deep neural networks and tree search},
  author={Silver, David and Huang, Aja and Maddison, Chris J and Guez, Arthur and Sifre, Laurent and Van Den Driessche, George and Schrittwieser, Julian and Antonoglou, Ioannis and Panneershelvam, Veda and Lanctot, Marc and others},
  journal={nature},
  volume={529},
  number={7587},
  pages={484--489},
  year={2016},
  publisher={Nature Publishing Group}
}

@inproceedings{liu2021finrl,
  title={FinRL: Deep reinforcement learning framework to automate trading in quantitative finance},
  author={Liu, Xiao-Yang and Yang, Hongyang and Gao, Jiechao and Wang, Christina Dan},
  booktitle={Proceedings of the second ACM international conference on AI in finance},
  pages={1--9},
  year={2021}
}

@article{parbhoo2017combining,
  title={Combining kernel and model based learning for HIV therapy selection},
  author={Parbhoo, Sonali and Bogojeska, Jasmina and Zazzi, Maurizio and Roth, Volker and Doshi-Velez, Finale},
  journal={AMIA Summits on Translational Science Proceedings},
  volume={2017},
  pages={239},
  year={2017}
}

@article{kuo2022health,
  title={The Health Gym: synthetic health-related datasets for the development of reinforcement learning algorithms},
  author={Kuo, Nicholas I-Hsien and Polizzotto, Mark N and Finfer, Simon and Garcia, Federico and S{\"o}nnerborg, Anders and Zazzi, Maurizio and B{\"o}hm, Michael and Kaiser, Rolf and Jorm, Louisa and Barbieri, Sebastiano},
  journal={Scientific data},
  volume={9},
  number={1},
  pages={693},
  year={2022},
  publisher={Nature Publishing Group UK London}
}

@article{wu2024quantifying,
  title={Quantifying Individual Risk for Binary Outcome: Bounds and Inference},
  author={Wu, Peng and Ding, Peng and Geng, Zhi and Liu, Yue},
  journal={arXiv preprint arXiv:2402.10537},
  year={2024}
}

@inproceedings{chen2019information,
  title={Information-theoretic considerations in batch reinforcement learning},
  author={Chen, Jinglin and Jiang, Nan},
  booktitle={International conference on machine learning},
  pages={1042--1051},
  year={2019},
  organization={PMLR}
}

@article{Chernozhukov2005,
author = {Victor Chernozhukov and Christian Hansen},
title = {An {I}{V} Model of Quantile Treatment Effects},
journal = {Econometrica},
volume = {73},
pages = {245-261},
year = {2005}}

@article{Heckman1997,
author = {James J. Heckman and Jeffrey Smith and Nancy Clements},
title = {Making the Most Out of Programme Evaluations and Social Experiments: Accounting
for Heterogeneity in Programme Impacts},
journal = {The Review of Economic Studies},
volume = {64},
pages = {487-535},
year = {1997}}

@article{wang2025counterfactually,
  title={Counterfactually Fair Reinforcement Learning via Sequential Data Preprocessing},
  author={Wang, Jitao and Shi, Chengchun and Piette, John D and Loftus, Joshua R and Zeng, Donglin and Wu, Zhenke},
  journal={arXiv preprint arXiv:2501.06366},
  year={2025}
}

@article{tropp2012user,
  title={User-friendly tail bounds for sums of random matrices},
  author={Tropp, Joel A},
  journal={Foundations of computational mathematics},
  volume={12},
  pages={389--434},
  year={2012},
  publisher={Springer}
}

@article{wang2020statistical,
  title={What are the statistical limits of offline RL with linear function approximation?},
  author={Wang, Ruosong and Foster, Dean P and Kakade, Sham M},
  journal={arXiv preprint arXiv:2010.11895},
  year={2020}
}

@inproceedings{munos2005error,
  title={Error bounds for approximate value iteration},
  author={Munos, R{\'e}mi},
  booktitle={Proceedings of the National Conference on Artificial Intelligence},
  volume={20},
  number={2},
  pages={1006},
  year={2005},
  organization={Menlo Park, CA; Cambridge, MA; London; AAAI Press; MIT Press; 1999}
}

@article{uehara2022review,
  title={A review of off-policy evaluation in reinforcement learning},
  author={Uehara, Masatoshi and Shi, Chengchun and Kallus, Nathan},
  journal={arXiv preprint arXiv:2212.06355},
  year={2022}
}

@article{bertagnolio2021clinical,
  title={Clinical impact of pretreatment human immunodeficiency virus drug resistance in people initiating nonnucleoside reverse transcriptase inhibitor--containing antiretroviral therapy: a systematic review and meta-analysis},
  author={Bertagnolio, Silvia and Hermans, Lucas and Jordan, Michael R and Avila-Rios, Santiago and Iwuji, Collins and Derache, Anne and Delaporte, Eric and Wensing, Annemarie and Aves, Theresa and Borhan, ASM and others},
  journal={The Journal of infectious diseases},
  volume={224},
  number={3},
  pages={377--388},
  year={2021},
  publisher={Oxford University Press US}
}

@article{de2010non,
  title={Non-nucleoside reverse transcriptase inhibitors (NNRTIs), their discovery, development, and use in the treatment of HIV-1 infection: a review of the last 20 years (1989--2009)},
  author={de B{\'e}thune, Marie-Pierre},
  journal={Antiviral research},
  volume={85},
  number={1},
  pages={75--90},
  year={2010},
  publisher={Elsevier}
}

@article{gunthard2016antiretroviral,
  title={Antiretroviral drugs for treatment and prevention of HIV infection in adults: 2016 recommendations of the International Antiviral Society--USA panel},
  author={G{\"u}nthard, Huldrych F and Saag, Michael S and Benson, Constance A and Del Rio, Carlos and Eron, Joseph J and Gallant, Joel E and Hoy, Jennifer F and Mugavero, Michael J and Sax, Paul E and Thompson, Melanie A and others},
  journal={Jama},
  volume={316},
  number={2},
  pages={191--210},
  year={2016},
  publisher={American Medical Association}
}

@article{antos2008learning,
  title={Learning near-optimal policies with Bellman-residual minimization based fitted policy iteration and a single sample path},
  author={Antos, Andr{\'a}s and Szepesv{\'a}ri, Csaba and Munos, R{\'e}mi},
  journal={Machine Learning},
  volume={71},
  number={1},
  pages={89--129},
  year={2008},
  publisher={Springer}
}

@article{agarwal2019reinforcement,
  title={Reinforcement learning: Theory and algorithms},
  author={Agarwal, Alekh and Jiang, Nan and Kakade, Sham M and Sun, Wen},
  journal={CS Dept., UW Seattle, Seattle, WA, USA, Tech. Rep},
  volume={32},
  pages={96},
  year={2019}
}

@article{zhong2025risk,
  title={Risk-sensitive deep rl: Variance-constrained actor-critic provably finds globally optimal policy},
  author={Zhong, Han and Deng, Xun and Fang, Ethan X and Yang, Zhuoran and Wang, Zhaoran and Li, Runze},
  journal={Journal of the American Statistical Association},
  number={just-accepted},
  pages={1--26},
  year={2025},
  publisher={Taylor \& Francis}
}

@article{wang2018quantile,
  title={Quantile-optimal treatment regimes},
  author={Wang, Lan and Zhou, Yu and Song, Rui and Sherwood, Ben},
  journal={Journal of the American Statistical Association},
  volume={113},
  number={523},
  pages={1243--1254},
  year={2018},
  publisher={Taylor \& Francis}
}

@article{viviano2024fair,
  title={Fair policy targeting},
  author={Viviano, Davide and Bradic, Jelena},
  journal={Journal of the American Statistical Association},
  volume={119},
  number={545},
  pages={730--743},
  year={2024},
  publisher={Taylor \& Francis}
}

@article{liu2025fairness,
  title={Fairness-aware contextual dynamic pricing with strategic buyers},
  author={Liu, Pangpang and Sun, Will Wei},
  journal={arXiv preprint arXiv:2501.15338},
  year={2025}
}

@article{fang2023fairness,
  title={Fairness-oriented learning for optimal individualized treatment rules},
  author={Fang, Ethan X and Wang, Zhaoran and Wang, Lan},
  journal={Journal of the American Statistical Association},
  volume={118},
  number={543},
  pages={1733--1746},
  year={2023},
  publisher={Taylor \& Francis}
}

@article{ouyang2022training,
  title={Training language models to follow instructions with human feedback},
  author={Ouyang, Long and Wu, Jeffrey and Jiang, Xu and Almeida, Diogo and Wainwright, Carroll and Mishkin, Pamela and Zhang, Chong and Agarwal, Sandhini and Slama, Katarina and Ray, Alex and others},
  journal={Advances in neural information processing systems},
  volume={35},
  pages={27730--27744},
  year={2022}
}

@article{singh1994upper,
  title={An upper bound on the loss from approximate optimal-value functions},
  author={Singh, Satinder P and Yee, Richard C},
  journal={Machine Learning},
  volume={16},
  number={3},
  pages={227--233},
  year={1994},
  publisher={Springer}
}

@article{chakraborty2013statistical,
  title={Statistical methods for dynamic treatment regimes},
  author={Chakraborty, Bibhas and Moodie, Erica E},
  journal={Springer-Verlag. doi},
  volume={10},
  number={978-1},
  pages={4--1},
  year={2013},
  publisher={Springer}
}

@book{puterman2014markov,
  title={Markov decision processes: discrete stochastic dynamic programming},
  author={Puterman, Martin L},
  year={2014},
  publisher={John Wiley \& Sons}
}

@article{panaganti2022robust,
  title={Robust reinforcement learning using offline data},
  author={Panaganti, Kishan and Xu, Zaiyan and Kalathil, Dileep and Ghavamzadeh, Mohammad},
  journal={Advances in neural information processing systems},
  volume={35},
  pages={32211--32224},
  year={2022}
}

@article{morris2016punishment,
  title={The punishment gap: School suspension and racial disparities in achievement},
  author={Morris, Edward W and Perry, Brea L},
  journal={Social Problems},
  volume={63},
  number={1},
  pages={68--86},
  year={2016},
  publisher={Oxford University Press}
}

@article{mowen2016school,
  title={School discipline as a turning point: The cumulative effect of suspension on arrest},
  author={Mowen, Thomas and Brent, John},
  journal={Journal of research in crime and delinquency},
  volume={53},
  number={5},
  pages={628--653},
  year={2016},
  publisher={Sage Publications Sage CA: Los Angeles, CA}
}

@article{wilt2017follow,
  title={Follow-up of prostatectomy versus observation for early prostate cancer},
  author={Wilt, Timothy J and Jones, Karen M and Barry, Michael J and Andriole, Gerald L and Culkin, Daniel and Wheeler, Thomas and Aronson, William J and Brawer, Michael K},
  journal={New England Journal of Medicine},
  volume={377},
  number={2},
  pages={132--142},
  year={2017},
  publisher={Mass Medical Soc}
}

@article{hu2025fast,
  title={Fast rates for the regret of offline reinforcement learning},
  author={Hu, Yichun and Kallus, Nathan and Uehara, Masatoshi},
  journal={Mathematics of Operations Research},
  volume={50},
  number={1},
  pages={633--655},
  year={2025},
  publisher={INFORMS}
}

@inproceedings{kumar2020conservative,
  title     = {Conservative Q-Learning for Offline Reinforcement Learning},
  author    = {Kumar, Aviral and Zhou, Aurick and Tucker, George and Levine, Sergey},
  booktitle = {Advances in Neural Information Processing Systems},
  volume    = {33},
  pages     = {1179--1191},
  year      = {2020}
}

@inproceedings{yu2021combo,
  title={COMBO: Conservative Offline Model-Based Policy Optimization},
  author={Yu, Tianhe and Kumar, Aviral and Chebotar, Yevgen and Hausman, Karol and Levine, Sergey and Finn, Chelsea},
  booktitle={Advances in Neural Information Processing Systems},
  volume={34},
  pages={28954--28967},
  year={2021}
}

@inproceedings{fujimoto2019off,
  title     = {Off-Policy Deep Reinforcement Learning without Exploration},
  author    = {Fujimoto, Scott and Meger, David and Precup, Doina},
  booktitle = {International Conference on Machine Learning},
  pages     = {2052--2062},
  year      = {2019},
  organization={PMLR}
}

@book{morath2005no,
  title={To do no harm: ensuring patient safety in health care organizations},
  author={Morath, Julianne M and Turnbull, Joanne E},
  year={2005},
  publisher={John Wiley \& Sons}
}

@article{Sun-etal2024,
	Author = {Shuo Sun and Johanna G. Nešlehová and Erica E. M. Moodie},
	Journal = {Statistics in medicine},
	Pages = {34--38},
	Publisher = {Wiley Online Library},
	Title = {Principal stratification for quantile causal effects under partial compliance},
	Volume = {43},
	Year = {2024}}

@article{ertefaie2018constructing,
  title={Constructing dynamic treatment regimes over indefinite time horizons},
  author={Ertefaie, Ashkan and Strawderman, Robert L},
  journal={Biometrika},
  volume={105},
  number={4},
  pages={963--977},
  year={2018},
  publisher={Oxford University Press}
}

@article{liao2022batch,
  title={Batch policy learning in average reward markov decision processes},
  author={Liao, Peng and Qi, Zhengling and Wan, Runzhe and Klasnja, Predrag and Murphy, Susan A},
  journal={Annals of statistics},
  volume={50},
  number={6},
  pages={3364},
  year={2022},
  publisher={NIH Public Access}
}

@article{shi2024statistically,
  title={Statistically efficient advantage learning for offline reinforcement learning in infinite horizons},
  author={Shi, Chengchun and Luo, Shikai and Le, Yuan and Zhu, Hongtu and Song, Rui},
  journal={Journal of the American Statistical Association},
  volume={119},
  number={545},
  pages={232--245},
  year={2024},
  publisher={Taylor \& Francis}
}

@article{luckett2020estimating,
  title={Estimating Dynamic Treatment Regimes in Mobile Health Using V-learning},
  author={Luckett, Daniel J and Laber, Eric B and Kahkoska, Anna R and Maahs, David M and Mayer-Davis, Elizabeth and Kosorok, Michael R},
  journal={Journal of the American Statistical Association},
  volume={115},
  number={530},
  pages={692},
  year={2020},
  publisher={NIH Public Access}
}

@article{kallus2022efficiently,
  title={Efficiently breaking the curse of horizon in off-policy evaluation with double reinforcement learning},
  author={Kallus, Nathan and Uehara, Masatoshi},
  journal={Operations Research},
  volume={70},
  number={6},
  pages={3282--3302},
  year={2022},
  publisher={INFORMS}
}

@inproceedings{joseph2016fairness,
  title={Fairness in learning: Classic and contextual bandits},
  author={Joseph, Matthew and Kearns, Michael and Morgenstern, Jamie and Roth, Aaron},
  booktitle={Advances in Neural Information Processing Systems},
  pages={325--333},
  year={2016}
}

@article{yang2023riskaverse,
  title={Risk-averse reinforcement learning via policy gradient},
  author={Yang, Lin and Xu, Huan and Zhang, Jia and Zhou, Zhi},
  journal={Journal of Machine Learning Research},
  volume={24},
  number={121},
  pages={1--39},
  year={2023}
}

@inproceedings{derman2018soft,
  title={Soft-robust actor-critic policy-gradient},
  author={Derman, Esther and Mannor, Shie},
  booktitle={Advances in Neural Information Processing Systems},
  pages={8307--8317},
  year={2018}
}

@misc{wang2020robust,
      title={Robust Reinforcement Learning with Wasserstein Constraint}, 
      author={Linfang Hou and Liang Pang and Xin Hong and Yanyan Lan and Zhiming Ma and Dawei Yin},
      year={2020},
      eprint={2006.00945},
      archivePrefix={arXiv preprint arXiv:2006.00945}, 
}

@inproceedings{pinto2017robust,
  title={Robust adversarial reinforcement learning},
  author={Pinto, Lerrel and Davidson, James and Sukthankar, Rahul and Gupta, Abhinav},
  booktitle={International Conference on Machine Learning},
  pages={2817--2826},
  year={2017}
}

@article{farahmand2011action,
  title={Action-gap phenomenon in reinforcement learning},
  author={Farahmand, Amir-massoud and Szepesv{\'a}ri, Csaba and Munos, R{\'e}mi},
  journal={Advances in Neural Information Processing Systems},
  volume={24},
  year={2011}
}

@article{tsybakov2004optimal,
  title={Optimal aggregation of classifiers in statistical learning},
  author={Tsybakov, Alexander B},
  journal={The Annals of Statistics},
  volume={32},
  number={1},
  pages={135--166},
  year={2004},
  publisher={Institute of Mathematical Statistics}
}

@article{audibert2007fast,
   title={Fast learning rates for plug-in classifiers},
   volume={35},
   ISSN={0090-5364},
   url={http://dx.doi.org/10.1214/009053606000001217},
   DOI={10.1214/009053606000001217},
   number={2},
   journal={The Annals of Statistics},
   publisher={Institute of Mathematical Statistics},
   author={Audibert, Jean-Yves and Tsybakov, Alexandre B.},
   year={2007}
}

@article{audibert2007tuning,
  title={Tuning bandit algorithms in stochastic environments},
  author={Audibert, Jean-Yves and Munos, R{\'e}mi and Szepesv{\'a}ri, Csaba},
  journal={Theoretical Computer Science},
  volume={410},
  number={19},
  pages={1876--1902},
  year={2009},
  publisher={Elsevier}
}

@inproceedings{kakade2002approximately,
  title={Approximately optimal approximate reinforcement learning},
  author={Kakade, Sham and Langford, John},
  booktitle={Proceedings of the nineteenth international conference on machine learning},
  pages={267--274},
  year={2002}
}

@inproceedings{zhang2018mitigating,
  title={Mitigating unwanted biases with adversarial learning},
  author={Zhang, Brian Hu and Lemoine, Bethany and Mitchell, Margaret},
  booktitle={Proceedings of the 2018 AAAI/ACM Conference on AI, Ethics, and Society},
  pages={335--340},
  year={2018}
}

@article{Jonsen1978,
 author = {Albert R. Jonsen},
 journal = {Annals of Internal Medicine},
 number = {6},
 pages = {827--832},
 publisher = {The MIT Press},
 title = {Do No Harm},
 volume = {88},
 year = {1978}
}

@article{Jin-etal2023,
	author = {Ying Jin and Zhimei Ren and Emmanuel J. Cand{\`e}s},
	journal = {Proceedings of the National Academy of Sciences},
	pages = {e2214889120},
	title = {Sensitivity analysis of individual treatment effects: A robust conformal inference approach},
	volume = {120},
	year = {2023}}

@inproceedings{Kallus2022,
author = {Kallus, Nathan},
title = {What's the harm? sharp bounds on the fraction negatively affected by treatment},
year = {2022},
booktitle = {Proceedings of the 36th International Conference on Neural Information Processing Systems},
articleno = {1164},
numpages = {14},
  pages={15996--16009},
}

@article{MuellerPearl2023-CausalInference,
	author = {Scott Mueller and Judea Pearl},
	journal = {Journal of Causal Inference},
	pages = {20220050},
	title = {Personalized Decision Making -- A Conceptual Introduction},
	volume = {11},
	year = {2023}}

@article{Wu-Mao2025,
  title={The Promises of Multiple Experiments: Identifying Joint Distribution of Potential Outcomes},
  author={Peng Wu and Xiaojie Mao},
  journal={arXiv preprint arXiv:2504.20470},
  year={2025}
}

@article{Wiens-etal2019,
 author = {Jenna Wiens and Suchi Saria and Mark Sendak and Marzyeh Ghassemi and Vincent X. Liu and Finale Doshi-Velez and Kenneth Jung and Katherine Heller and David Kale and Mohammed Saeed and  Pilar N. Ossorio and Sonoo Thadaney-Israni and Anna Goldenberg},
 journal = {Nature Medicine}, 
 title = {Do no harm: a roadmap for responsible machine learning for health care},
 volume = {25},
 number = {10}, 
 pages = {1337--1340},
 year = {2019}
}

@inproceedings{chow2018lyapunov,
  title={Lyapunov-based safe policy optimization for continuous control},
  author={Chow, Yinlam and Nachum, Ofir and Duenez-Guzman, Edgar and Ghavamzadeh, Mohammad},
  booktitle={Advances in Neural Information Processing Systems},
  pages={8103--8112},
  year={2018}
}

@article{tang2020worst,
  title={Worst-case aware policy optimization for robust reinforcement learning},
  author={Tang, Ziyu and Zhang, Qinqing and Sun, Wen},
  journal={arXiv preprint arXiv:2002.08033},
  year={2020}
}

@article{thomas2019preventing,
author = {Philip S. Thomas  and Bruno Castro da Silva  and Andrew G. Barto  and Stephen Giguere  and Yuriy Brun  and Emma Brunskill },
title = {Preventing undesirable behavior of intelligent machines},
journal = {Science},
volume = {366},
number = {6468},
pages = {999-1004},
year = {2019}}

@article{jiang2022constrained,
  title={Constrained policy optimization via approximate Lagrangian methods},
  author={Jiang, Sheng and Wang, Kaiqing and Yang, Zhuoran and Wang, Mengdi},
  journal={SIAM Journal on Optimization},
  volume={32},
  number={1},
  pages={147--177},
  year={2022}
}

@article{shi2023risk,
  title={Risk-sensitive offline reinforcement learning via uncertainty-regularized policy optimization},
  author={Shi, Xiaojie and Yang, Lin and Zhou, Zhi},
  journal={Machine Learning},
  volume={112},
  number={4},
  pages={1389--1421},
  year={2023},
  publisher={Springer}
}

@article{Lei-Candes2021,
	author = {Lihua Lei and Emmanuel J. Cand{\`e}s},
	journal = {Journal of the Royal Statistical Society: Series B (Statistical Methodology)},
	pages = {911-938},
	title = {Conformal Inference of Counterfactuals and Individual Treatment Effects},
	volume = {83},
	year = {2021}}

@book{Hernan-Robins2020,
	author = {M.A. Hern{\'a}n and J. M. Robins},
	publisher = {Boca Raton: Chapman and Hall/CRC},
	title = {Causal Inference: What If},
	year = {2020}}

@article{Bartolucci01062011,
author = {Francesco Bartolucci and Leonardo Grilli},
title = {Modeling Partial Compliance Through Copulas in a Principal Stratification Framework},
journal = {Journal of the American Statistical Association},
volume = {106},
number = {494},
pages = {469--479},
year = {2011}
}

@article{kusner2017counterfactual,
	author = {Kusner, Matt J and Loftus, Joshua and Russell, Chris and Silva, Ricardo},
	journal = {Advances in neural information processing systems},
	title = {Counterfactual fairness},
	volume = {30},
	year = {2017}}

@book{Jaworski-etal-compula2010,
	author = {P. Jaworski and F. Durante and W. K. Hardle and T. Rychlik},
	publisher = {Springer},
	title = { Copula Theory and Its Applications},
	volume = {198},
	year = {2010}}

@book{pearl2009causality,
	author = {Pearl, Judea},
	publisher = {Cambridge university press},
	title = {Causality},
	year = {2009}}

@article{Stone1982,
author = {Charles J. Stone},
title = {Optimal Global Rates of Convergence for Nonparametric Regression},
journal = {The Annals of Statistics},
volume = {10},
number = {4},
pages = {1040--1053},
year = {1982}
}

@book{Fan1996,
  title={Local Polynomial Modelling and Its Applications},
  author={Fan, Jianqing and Gijbels, Ir{\`e}ne},
  year={1996},
  publisher={Chapman and Hall/CRC},
  address={London}
}

@book{Tsybakov2009,
  title={Introduction to Nonparametric Estimation},
  author={Tsybakov, Alexandre B.},
  year={2009},
  publisher={Springer Science \& Business Media},
  address={New York}
}

@book{Gyorfi2002,
  title={A Distribution-Free Theory of Nonparametric Regression},
  author={Gy{\"o}rfi, L{\'a}szl{\'o} and Kohler, Michael and Krzyzak, Adam and Walk, Harro},
  year={2002},
  publisher={Springer Science \& Business Media},
  address={New York}
}

@article{gottesman2019guidelines,
  title={Guidelines for reinforcement learning in healthcare},
  author={Gottesman, Omer and Johansson, Fredrik and Komorowski, Matthieu and Faisal, Aldo and Sontag, David and Doshi-Velez, Finale and Celi, Leo Anthony},
  journal={Nature medicine},
  volume={25},
  number={1},
  pages={16--23},
  year={2019},
  publisher={Nature Publishing Group}
}

@article{shortreed2011informing,
  title={Informing sequential clinical decision-making through reinforcement learning: an application to dyslipidemia},
  author={Shortreed, Susan M and Laber, Eric and Lizotte, Daniel J and Stroup, T Scott and Pineau, Joelle and Murphy, Susan A},
  journal={Statistical Methods in Medical Research},
  volume={20},
  number={6},
  pages={631--659},
  year={2011},
  publisher={SAGE Publications Sage UK: London, England}
}

 \newpage

  \begin{center}
\bf \Large 
Supplementary Material
\end{center}

\setcounter{equation}{0}
\setcounter{section}{0}
\setcounter{figure}{0}
\setcounter{example}{0}
\setcounter{proposition}{0}
\setcounter{corollary}{0}
\setcounter{theorem}{0}
\setcounter{table}{0}
\setcounter{condition}{0}
\setcounter{lemma}{0}
\setcounter{remark}{0}

\renewcommand {\theproposition} {S\arabic{proposition}}
\renewcommand {\theexample} {S\arabic{example}}
\renewcommand {\thefigure} {S\arabic{figure}}
\renewcommand {\thetable} {S\arabic{table}}
\renewcommand {\theequation} {S\arabic{equation}}
\renewcommand {\thelemma} {S\arabic{lemma}}
\renewcommand {\thesection} {S\arabic{section}}
\renewcommand {\thetheorem} {S\arabic{theorem}}
\renewcommand {\thecorollary} {S\arabic{corollary}}
\renewcommand {\thecondition} {S\arabic{condition}}
\renewcommand {\thepage} {S\arabic{page}}
\renewcommand {\theremark} {S\arabic{remark}}

\setcounter{page}{1}

  \setcounter{equation}{0}
\renewcommand {\theequation} {S\arabic{equation}}
  \setcounter{lemma}{0}
\renewcommand {\thelemma} {S\arabic{lemma}}
   \setcounter{definition}{0}
\renewcommand {\thedefinition} {S\arabic{definition}}
   \setcounter{example}{0}
\renewcommand {\theexample} {S\arabic{example}}
   \setcounter{proposition}{0}
\renewcommand {\theproposition} {S\arabic{proposition}}
   \setcounter{corollary}{0}
\renewcommand {\thecorollary} {S\arabic{corollary}}

 \bigskip 
 

\section{Harm Rate Identifiability under a Gaussian Copula}

In this section, we derive the identifiability formula for the harm rate presented at the end of Section \ref{sec:identi} of the manuscript.  
Under the Gaussian copula assumption with parameter $\rho$, the joint distribution of $(Y_t(a), Y_t(a'))$ given $X_t=x$ is a bivariate normal distribution:
\[
    \begin{pmatrix} Y_t(a) \\ Y_t(a') \end{pmatrix} \Bigg| X_t=x \sim \mathcal{N} \left( \begin{pmatrix} r(x,a) \\ r(x,a') \end{pmatrix}, \begin{pmatrix} \sigma^2(x,a) & \rho \sigma(x,a)\sigma(x,a') \\ \rho \sigma(x,a)\sigma(x,a') & \sigma^2(x,a') \end{pmatrix} \right).
\]
Let $\Delta Y_t = Y_t(a) - Y_t(a')$ be the difference in outcomes. Since $\Delta Y_t$ is a linear combination of jointly normal variables, it follows a univariate normal distribution:
\[
    \Delta Y_t \mid X_t=x \sim \mathcal{N}\left( \Delta_r(x), \sigma(x; a, a') \right),
\]
where $\Delta_r(x) = r(x,a) - r(x,a')$ with
$r(x, a) = \E[ Y_t \mid X_t = x, A_t=a]$, and 
$\sigma(x; a, a') = \sqrt{\sigma^2(x,a) + \sigma^2(x,a')  - 2 \rho \sigma(x,a) \sigma(x,a')}$ with $\sigma^2(x,a) = \text{Var}(Y_t\mid X_t = x, A_t = a)$. The harm rate is the probability that the action $a$ leads to a worse outcome than $a'$:
\begin{align*}
    \HR(x; a, a') &= \Pr(Y_t(a) - Y_t(a') < 0 \mid X_t=x) \\
    &= \Pr\left( \frac{\Delta Y_t - \Delta_r(x)}{\sigma(x; a, a')} < \frac{0 - \Delta_r(x)}{\sigma(x; a, a')} \right) \\
    &= \Phi\left( -\frac{\Delta_r(x)}{\sigma(x; a, a')} \right) \\
    &= \Phi\left( \frac{r(x,a') - r(x,a)}{\sqrt{\sigma^2(x,a) + \sigma^2(x,a') - 2\rho \sigma(x,a)\sigma(x,a')}} \right).
\end{align*}
This confirms the identifiability formula presented in the main text.

We could use the plug-in method for estimating $\HR(x;a, a')$. The corresponding estimation procedure consists of two steps. (1) Estimate $r(x,a)$ and $\sigma (x,a)$. 
    we could estimate $r(x, a)$ by directly regressing $Y$ on $X$ using data with $A = a$, and estimate $ \sigma^2(x,a)$ by regressing $(Y-\hat r(X_t, a))^2$ on $X_t$ using the data with $A=a$, denoted as $\hat r(x,a)$ and $\hat \sigma(x, a)$. (2) Estimate $\HR(x;a, a')$ using the plug-in method. 

\section{Extension to Harm Value}
\subsection{Formalization}
Similar to Equation \eqref{eq:R-hr} in the manuscript, we can formulate the utility function by incorporating the Harm Value (HQ) to account for the \textit{severity} of the potential harm, rather than just its probability. In this case, we define the utility at time $t$ under action $a$ as:
\begin{equation} \label{eq:R-hq}
    R_t(a) = Y_t(a) - \beta \cdot \HQ(X_t; a, a'),
\end{equation}
where $\beta > 0$ remains the risk-aversion parameter. Here, $\HQ(X_t; a, a')$ represents the expected magnitude of the adverse effect (excess loss) caused by action $a$ relative to the reference action $a'$. The definitions of the value function $V^\pi$, the action-value function $Q^\pi$, and the overall objective $J^*$ follow the same structure as in the manuscript, with the HR-based utility replaced by this HQ-based utility.

\subsection{Identifiability}
For the Gaussian copula, the joint distribution of potential outcomes is fully specified by the marginals and the correlation coefficient. The Harm Value is identifiable under Assumptions \ref{assump1}–\ref{assump-copula} in the manuscript and admits a closed-form expression. Specifically, the identifiability formula for $\HQ(x; a, a')$ is given by 
\begin{equation} \label{eq-s2} \begin{split}
    \HQ(x;a, a') =& \left(r(x, a')-r(x, a)\right)\Phi\left(\dfrac{r(x,a')-r(x,a)}{\sigma(x; a, a')}\right)\\  +&\dfrac{\sigma(x; a, a')}{\sqrt{2\pi}}\exp\left\{-\dfrac{(r(x,a)-r(x,a'))^2}{2\sigma(x; a, a')^2}\right\}. 
    \end{split}
\end{equation} 
Also, we may employ a plug-in approach to estimate $\HQ(x; a, a')$. Specifically, we first estimate $r(x, a)$ and $\sigma(x, a)$, and then obtain $\HQ(x; a, a')$ by substituting these estimates into the corresponding expression.

\medskip 
\noindent
\emph{Proof of Equation \eqref{eq-s2}}. 
Recall that under a Gaussian copula model,  $\Delta Y_t = Y_t(a) - Y_t(a')$ follows a univariate normal distribution:
$
    \Delta Y_t \mid X_t=x \sim \mathcal{N}\left( \Delta_r(x), \sigma(x; a, a') \right),
$
where  
$\Delta_r(x) = r(x,a) - r(x,a')$ and
$\sigma(x; a, a') = \sqrt{\sigma^2(x,a) + \sigma^2(x,a')  - 2 \rho \sigma(x,a) \sigma(x,a')}$.
Using the properties of the truncated normal distribution, we have: 
\begin{align*}
    \HQ(x; a, a') &= \E\left[ (Y_t(a') - Y_t(a)) \cdot \I(Y_t(a) < Y_t(a')) \mid X_t=x \right] \\
    &= -\E\left[ \Delta Y_t \cdot \I(\Delta Y_t < 0) \right] \\
    &= -\int_{-\infty}^{0} r \cdot \frac{1}{\sqrt{2\pi}\sigma(x; a, a') } \exp\left\{ -\frac{(r - \Delta_r(x))^2}{2\sigma^2(x; a, a') } \right\} \mathrm{d}r.
\end{align*}
Let $z = (r - \Delta_r(x))/\sigma(x; a, a')$. The integral transforms to: 
\begin{align*}
    \HQ(x; a, a') &= -\int_{-\infty}^{-\frac{\Delta_r(x)}{\sigma(x; a, a') }} (\Delta_r(x) + \sigma(x; a, a')  z) \phi(z) \mathrm{d}z \\
    &= -\Delta_r(x) \Phi\left(-\frac{\Delta_r(x)}{\sigma(x; a, a')}\right) + \sigma(x; a, a')  \phi\left(-\frac{\Delta_r(x)}{\sigma(x; a, a') }\right) \\
    &= (r(x,a') - r(x,a)) \Phi\left( \frac{r(x,a') - r(x,a)}{\sigma(x; a, a') } \right) \\
    & \quad + \sigma(x; a, a') \phi\left( \frac{r(x,a') - r(x,a)}{\sigma(x; a, a')} \right),
\end{align*}
where $\phi(\cdot)$ is the probability density function of the standard normal distribution.

$\hfill \Box$

Next, we briefly discuss the connection between the Gaussian copula model and the additive heteroscedastic Gaussian noise model.
The additive heteroscedastic Gaussian noise model, defined in the following Assumption \ref{ass:additive}, implies a specific dependence structure between the potential outcomes. 

\begin{assumption}[Additive Heteroscedastic Gaussian Noise]\label{ass:additive}
Suppose that $X_t, A_t$ and $Y_t$ satisfy the following additive gaussian noise model: 
\begin{equation}\label{eq:additive}
\left\{
\begin{aligned}
&Y_t = r(X_t, A_t) + \sigma(X_t, A_t) \epsilon_t \\
&\epsilon_t  \sim \mathcal{N}(0, 1),~  \sigma(X_t, A_t) > 0,
\end{aligned}
\right.
\end{equation}
where $\epsilon_t$ is the noise variable. 
\end{assumption}

Under Assumptions \ref{ass:additive}, $\E[Y_t(a) | X_t=x] = r(x, a)$ and $\var[Y_t(a) | X_t=s]=\sigma^2(x, a)$. Thus the two terms $r(s, a)$ and $\sigma(s, a)$ represent the conditional mean and standard deviation of $Y_t(a)$. 
Notably, the additive heteroscedastic model in~\eqref{eq:additive} is  flexible, as it places no restrictions on the functional forms of the conditional mean $r(X_t, A_t)$ or the noise structure $\sigma(X_t, A_t)$. Both components can vary with the state and action, allowing for rich heterogeneity across individuals and treatments.
In addition, the model \eqref{eq:additive} also permits the noise $\epsilon_t$ to depend on the current state $X_t$, 
 see Remark \ref{rmk1} below for more details. 
\begin{remark}[$\epsilon_t$ depends on $X_t$] \label{rmk1} 
Suppose that the true model is  
\begin{equation} \label{eq:additive_ex}
        Y_t = r(X_t, A_t) + \sigma(X_t, A_t) \epsilon_t, ~\epsilon_t \mid X_t \sim \mathcal{N}\left(M_t(X_t), D_t^2(X_t)\right), 
\end{equation}
where $\sigma(X_t, A_t)>0, \ D(X_t) > 0$.
In model \eqref{eq:additive_ex}, the conditional standard normal is violated. Nevertheless, we could write $\epsilon_t$ as $M_t(X_t) + D_t(X_t) \tilde \epsilon_t$ (i.e., let $\tilde \epsilon_t = (\epsilon_t - M_t(X_t))/D_t(X_t)$), where $\tilde \epsilon_t\mid X_t \sim N(0, 1)$. Let $\tilde r(X_t, A_t) = r(X_t, A_t) + \sigma(X_t, A_t)M_t(X_t)$ and $\tilde \sigma(X_t, A_t) = \sigma(X_t, A_t) D_t(X_t)$, then model \eqref{eq:additive_ex} can be reformulated as the form of model \eqref{eq:additive}. 
\end{remark}

Despite the flexibility of the additive heteroscedastic Gaussian noise model, it corresponds to a special case of the Gaussian copula with correlation coefficient $\rho = 1$. Since the noise term $\epsilon_t$ is shared across different actions, it is not hard to verify that the joint distribution of $(Y(a), Y(a'))$ conditional on $X_t$ is Gaussian with correlation coefficient $\rho = 1$, a property commonly referred to as rank preservation~\citep{Heckman1997, Chernozhukov2005}.  
\section{Additional Numerical Results}

\paragraph{Implementation details}
For offline policy learning, we employ the Fitted Q Iteration (FQI) algorithm. The Q-function is approximated by a Multi-layer Perceptron (MLP) with a single hidden layer of 64 units. The network is optimized using the Adam optimizer with a learning rate of 0.001 and a batch size of 64. We train the model for 50 epochs using the mean squared error (MSE) as the loss function. To ensure training stability, the target network parameters are updated every 10 epochs. A discount factor of $\gamma = 0.9$ is applied throughout the experiments. The training process includes a convergence check, terminating early if the sum of absolute differences between the Q-network and target network parameters falls below a threshold of $10^{-5}$. All computations were performed on a local machine equipped with an Apple M3 Pro chip and 18 GB of unified memory. 
For Assumption~\ref{assump-copula}, we adopt a Gaussian copula with parameter $\rho$ in both the simulation and application settings. In the simulation study, we set $\rho = 1$, whereas in the application analysis, $\rho$ is treated as a sensitivity parameter and varied over the values $0$, $0.5$, and $1$.

\paragraph{Numerical details}
To complement the experimental results in the main text, we provide a comprehensive breakdown of the performance of our algorithm across different dataset sizes, with the values of $N \in \{100, 500, 1000, 2000\}$. Table \ref{tab:N_linear} and \ref{tab:N_nonlinear} report the detailed numerical values for discounted rewards and average harms over 100 replications, under the linear and nonlinear settings, respectively. The results show that as $N$ increases, the standard deviation narrows down, further supporting the consistency properties claimed in the manuscript.

\begin{table}[t]
    \centering
        \caption{Performance comparison across different sample sizes ($N$) under linear setting.}
        \setlength{\tabcolsep}{6pt}
    \begin{tabular}{lcccc}
        \toprule
        \textbf{Sample Size} & \textbf{Unaware} & \textbf{Harm-Aware} & \textbf{Behavior} & \textbf{Random} \\
        \midrule
        \multicolumn{5}{c}{\textbf{Discounted Reward}} \\
        \midrule
        $N=100$  & $2.310 \pm 0.076$ & $2.271 \pm 0.083$ & $1.246 \pm 0.066$ & $1.549 \pm 0.061$ \\
        $N=500$  & $2.303 \pm 0.035$ & $2.264 \pm 0.037$ & $1.246 \pm 0.031$ & $1.545 \pm 0.026$ \\
        $N=1000$ & $2.304 \pm 0.025$ & $2.264 \pm 0.027$ & $1.249 \pm 0.017$ & $1.548 \pm 0.022$ \\
        $N=2000$ & $2.303 \pm 0.018$ & $2.264 \pm 0.020$ & $1.247 \pm 0.016$ & $1.549 \pm 0.016$ \\
        \midrule
        \multicolumn{5}{c}{\textbf{Average Harm}} \\
        \midrule
        $N=100$  & $0.126 \pm 0.010$ & $0.071 \pm 0.009$ & $0.134 \pm 0.014$ & $0.135 \pm 0.013$ \\
        $N=500$  & $0.125 \pm 0.006$ & $0.070 \pm 0.005$ & $0.132 \pm 0.006$ & $0.134 \pm 0.006$ \\
        $N=1000$ & $0.126 \pm 0.005$ & $0.070 \pm 0.004$ & $0.133 \pm 0.005$ & $0.134 \pm 0.004$ \\
        $N=2000$ & $0.126 \pm 0.005$ & $0.070 \pm 0.003$ & $0.133 \pm 0.003$ & $0.134 \pm 0.003$ \\
        \bottomrule
    \end{tabular}
    \label{tab:N_linear}
\end{table}

\begin{table}[t]
    \centering
    \caption{Performance comparison across different sample sizes ($N$) under nonlinear setting.}
    \setlength{\tabcolsep}{6pt}
    \begin{tabular}{lcccc}
        \toprule
        \textbf{Sample Size} & \textbf{Unaware} & \textbf{Harm-Aware} & \textbf{Behavior} & \textbf{Random} \\
        \midrule
        \multicolumn{5}{c}{\textbf{Discounted Reward}} \\
        \midrule
        $N=100$  & $3.781 \pm 0.098$ & $3.396 \pm 0.149$ & $2.788 \pm 0.119$ & $2.796 \pm 0.116$ \\
        $N=500$  & $3.775 \pm 0.072$ & $3.390 \pm 0.143$ & $2.772 \pm 0.047$ & $2.789 \pm 0.051$ \\
        $N=1000$ & $3.780 \pm 0.070$ & $3.393 \pm 0.143$ & $2.783 \pm 0.038$ & $2.799 \pm 0.031$ \\
        $N=2000$ & $3.779 \pm 0.068$ & $3.393 \pm 0.143$ & $2.780 \pm 0.027$ & $2.793 \pm 0.027$ \\
        \midrule
        \multicolumn{5}{c}{\textbf{Average Harm}} \\
        \midrule
        $N=100$  & $0.096 \pm 0.024$ & $0.036 \pm 0.007$ & $0.098 \pm 0.005$ & $0.079 \pm 0.004$ \\
        $N=500$  & $0.096 \pm 0.024$ & $0.036 \pm 0.007$ & $0.099 \pm 0.002$ & $0.080 \pm 0.002$ \\
        $N=1000$ & $0.096 \pm 0.024$ & $0.036 \pm 0.007$ & $0.099 \pm 0.001$ & $0.079 \pm 0.001$ \\
        $N=2000$ & $0.096 \pm 0.024$ & $0.036 \pm 0.007$ & $0.099 \pm 0.001$ & $0.080 \pm 0.001$ \\
        \bottomrule
    \end{tabular}
    \label{tab:N_nonlinear}
\end{table}

\begin{table}[t]
    \centering
        \caption{Performance comparison across different penalty parameters ($\beta$) under linear setting.}
        \setlength{\tabcolsep}{6pt}
    \begin{tabular}{lcccc}
        \toprule
        \textbf{Penalty} & \textbf{Unaware} & \textbf{Harm-Aware} & \textbf{Behavior} & \textbf{Random} \\
        \midrule
        \multicolumn{5}{c}{\textbf{Discounted Reward} } \\
        \midrule
        $\beta=0.1$ & $2.304 \pm 0.002$ & $2.264 \pm 0.003$ & $1.249 \pm 0.002$ & $1.548 \pm 0.002$ \\
        $\beta=0.3$ & $2.304 \pm 0.002$ & $2.060 \pm 0.004$ & $1.249 \pm 0.002$ & $1.548 \pm 0.002$ \\
        $\beta=0.5$ & $2.304 \pm 0.002$ & $1.712 \pm 0.015$ & $1.249 \pm 0.002$ & $1.548 \pm 0.002$ \\
        $\beta=0.7$ & $2.304 \pm 0.002$ & $1.402 \pm 0.028$ & $1.249 \pm 0.002$ & $1.548 \pm 0.002$ \\
        $\beta=0.9$ & $2.304 \pm 0.002$ & $0.981 \pm 0.031$ & $1.249 \pm 0.002$ & $1.548 \pm 0.002$ \\
        \midrule
        \multicolumn{5}{c}{\textbf{Average Harm}} \\
        \midrule
        $\beta=0.1$ & $0.126 \pm 0.001$ & $0.070 \pm 0.000$ & $0.133 \pm 0.001$ & $0.134 \pm 0.000$ \\
        $\beta=0.3$ & $0.126 \pm 0.001$ & $0.047 \pm 0.000$ & $0.133 \pm 0.001$ & $0.134 \pm 0.000$ \\
        $\beta=0.5$ & $0.126 \pm 0.001$ & $0.043 \pm 0.000$ & $0.133 \pm 0.001$ & $0.134 \pm 0.000$ \\
        $\beta=0.7$ & $0.126 \pm 0.001$ & $0.043 \pm 0.000$ & $0.133 \pm 0.001$ & $0.134 \pm 0.000$ \\
        $\beta=0.9$ & $0.126 \pm 0.001$ & $0.043 \pm 0.000$ & $0.133 \pm 0.001$ & $0.134 \pm 0.000$ \\
        \bottomrule
    \end{tabular}
    \label{tab:beta_linear}
\end{table}

\begin{table}[t]
    \centering
    \caption{Performance comparison across different penalty parameters ($\beta$) under nonlinear setting.}
    \setlength{\tabcolsep}{6pt}
    \begin{tabular}{lcccc}
        \toprule
        \textbf{Penalty} & \textbf{Unaware} & \textbf{Harm-Aware} & \textbf{Behavior} & \textbf{Random} \\
        \midrule
        \multicolumn{5}{c}{\textbf{Discounted Reward} } \\
        \midrule
        $\beta=0.60$ & $3.780 \pm 0.007$ & $3.741 \pm 0.011$ & $2.783 \pm 0.004$ & $2.799 \pm 0.003$ \\
        $\beta=0.70$ & $3.780 \pm 0.007$ & $3.572 \pm 0.015$ & $2.783 \pm 0.004$ & $2.799 \pm 0.003$ \\
        $\beta=0.80$ & $3.780 \pm 0.007$ & $3.425 \pm 0.017$ & $2.783 \pm 0.004$ & $2.799 \pm 0.003$ \\
        $\beta=0.90$ & $3.780 \pm 0.007$ & $3.380 \pm 0.017$ & $2.783 \pm 0.004$ & $2.799 \pm 0.003$ \\
        \midrule
        \multicolumn{5}{c}{\textbf{Average Harm}} \\
        \midrule
        $\beta=0.60$ & $0.096 \pm 0.002$ & $0.067 \pm 0.001$ & $0.099 \pm 0.000$ & $0.079 \pm 0.000$ \\
        $\beta=0.70$ & $0.096 \pm 0.002$ & $0.048 \pm 0.001$ & $0.099 \pm 0.000$ & $0.079 \pm 0.000$ \\
        $\beta=0.80$ & $0.096 \pm 0.002$ & $0.038 \pm 0.001$ & $0.099 \pm 0.000$ & $0.079 \pm 0.000$ \\
        $\beta=0.90$ & $0.096 \pm 0.002$ & $0.035 \pm 0.001$ & $0.099 \pm 0.000$ & $0.079 \pm 0.000$ \\
        \bottomrule
    \end{tabular}
    \label{tab:beta_nonlinear}
\end{table}

Table \ref{tab:beta_linear} and \ref{tab:beta_nonlinear} display both the discounted reward and average harm across a wide range of $\beta$ values, under the cases of linear and nonlinear. These results illustrate the trade-off between maximizing utility and minimizing risk as the safety penalty increases. Results show that increasing $\beta$ enforces stricter safety constraints and leads to a decrease in Average Harm. This confirms our expectations discussed in the main text. Although a higher $\beta$ results in a lower Discounted Reward, it demonstrates that our method effectively controls risk. Practitioners can adjust $\beta$ to meet specific safety requirements.

\section{An Alternative Formulation of the Target Policy}

Recall that the target policy $\pi^*$ maximizes $J(\pi) = J_Y(\pi) -  \beta \cdot J_{\HR}(\pi)$, and $\bar{\pi}^*$  is defined as the solution to
 \begin{align*} 
 \begin{cases}
 \max_{\pi} & J_Y(\pi) \\
  \text{s.t.} &   J_{\HR}(\pi) \le \delta.
 \end{cases}
 \end{align*} 

The following Proposition establishes the connection between $\pi^*$ and  $\bar{\pi}^*$. 

\begin{proposition}
\label{thm:beta}
For a given $\beta$, the target policy $\pi^*$ coincides with $\bar{\pi}^*$ provided that 
 \[   \delta =  \E[ \sum_{t=0}^{\infty}  \gamma^t   \HR(X_t; \pi_t^*(X_t), a') ]. \]
\end{proposition}

Theorem \ref{thm:beta} reveals that the unconstrained optimization problem ($\max_{\pi} J(\pi)$) could be reformulated as a constrained problem that maximizes discounted cumulative outcome while keeping the harm rate induced by the learned policy below a given threshold. 

\noindent
{\bf Proof of Proposition \ref{thm:beta}}.  
For the constrained optimization problem, we show that the optimal policy $\bar \pi^*$, denoted as $\bar\pi^* = (\bar\pi_1^*, ..., \bar\pi_T^*)$ 
must satisfy 
\begin{equation}  \label{eq-S12}
\begin{split}
    J_{\text{HR}}(\bar\pi^*) 
={}& \E^{\bar\pi^*}\Big[ \sum_{t=0}^{\infty} \gamma^t \text{HR}_t(X_t; A_t, a')  \Big] = \E[ \sum_{t=0}^{\infty}  \gamma^t   \HR(X_t; \bar \pi_t^*(X_t), a') ]  =
\delta,  
\end{split} 
\end{equation} 
that is, the optimal solution is obtained at the boundary of the constraint.

We first partition the trajectories into four groups according to the discounted cumulative outcome and the discounted cumulative harm rate induced by $\bar \pi^*$, as shown in Table \ref{Tab-S3}. For groups $\mathcal{G}_2$ and $\mathcal{G}_4$, $\bar \pi^*$ fails to increase discounted cumulative outcome and instead introduces discounted cumulative harm.  
In contrast,  for group $\mathcal{G}_3$, $\bar \pi^*$ increases discounted cumulative outcome without causing discounted cumulative harm. Only for  group $\mathcal{G}_1$, $\bar \pi^*$ involves a tradeoff between discounted cumulative outcome and harm. 

   \begin{table}[h!] 
\centering
  \caption{Classification of trajectories, where $Y_{t}(\bar\pi_t^*(X_t)) :=  \sum_{a} \gamma^t \mathbb{I}( \bar\pi_t^*(X_t)  = a ) Y_t(a)$.} 
  \begin{tabular}{cccc}
    \toprule
                 &  $\sum_{t=0}^{\infty} \gamma^t \E[ Y_{t}(\pi_t^*(X_t)) \mid X_{t} ]  > 0$ &  $\sum_{t=0}^{\infty} \gamma^t \E[ Y_{t}(\pi_t^*(X_t)) \mid X_{t} ]  \leq 0 $ \\
    \midrule 
 $\sum_{t=0}^{\infty}  \gamma^t   \HR(X_t; \bar\pi_t^*(X_t), a')  > 0$  & $\mathcal{G}_1$ & $\mathcal{G}_2$ \\
    $\sum_{t=0}^{\infty}  \gamma^t   \HR(X_t; \bar\pi_t^*(X_t), a') = 0 $  & $\mathcal{G}_3$ & $\mathcal{G}_4$ \\
    \bottomrule
  \end{tabular}
  \label{Tab-S3}  
\end{table}

Next, we prove \eqref{eq-S12} using the method of contradiction. 
Specifically, if we assume that the optimal policy is not obtained at the boundary of the constraint, i.e., 
\[  
    J_{\text{HR}}(\bar \pi^*)  < \delta.  \] 
When $J_{\text{HR}}(\bar \pi^*)  < \delta$, there are some trajectories in the group $\mathcal{G}_1$ who are not assigned the action with positive  $\sum_{t=0}^{\infty} \gamma^t \E[ Y_{t}(\bar\pi_t^*(X_t)) \mid X_{t} ]$. Thus, we could find another policy, denoted as $\tilde \pi^*$,  
  that assigns more progressive action to some of the trajectories in $\mathcal{G}_1$ compared with $\bar\pi^*$. In contrast to optimal policy $\bar\pi^*$,  $\tilde \pi^*$ yields a higher discounted cumulative outcome but increases the discounted cumulative harm. That is, the policy $\tilde \pi^*$ will lead to a discounted cumulative harm rate closer to $\delta$ but will have a higher discounted cumulative outcome than  $\bar\pi^*$; thus, $\bar\pi^*$ is not the optimal policy, which contradicts its definition of $\bar\pi^*$.   
Therefore,  
\begin{align*}
\bar\pi^*  
    = \begin{cases}
    \displaystyle 
    \arg\max_{\pi}~  J_Y(\pi), \\
    \text{s.t. } ~
     J_{\HR}(\pi) = \delta.
    \end{cases}
\end{align*}

Then, by introducing the Lagrange multiplier $\beta$ (we set it equal to the risk-aversion factor $\beta$ and then $\bar\pi^*  = \pi^*$), $\bar\pi^*$ satisfies 
\[ \begin{cases}
\displaystyle
\arg \max_{\pi} \quad J_Y(\pi) -  \beta \cdot  J_{\HR}(\pi)  \\
s.t. \quad   J_{\HR}(\pi)  = \delta.  
\end{cases} \]  
This finishes the proof. 

\hfill $\Box$

\newpage 
\section{Additional Technical Proofs}

\subsection{Proof of Theorem \ref{thm:regret}}

The following Lemmas \ref{lem:HuTh8}–\ref{lem:ChLe13} will be used in the proof of Theorem \ref{thm:regret}. 

\begin{lemma}\label{lem:HuTh8}
    Under the completeness assumption \ref{ass:linear}(ii) and feature convergence assumption \ref{ass:linear}(iii) in the manuscript, and assume $|\hat{r}_t|\leq M$, we have 
    \begin{align}\label{eq:Q}
    \|Q^*- \hat Q_K\|_{\infty} \le \sum_{t=0}^{K-1} \gamma^t\|\hat Q_{K-t} - \mathcal{T} \hat Q_{K-t-1}\|_{\infty} +  \frac{\gamma^K M}{1-\gamma}.
\end{align}
\end{lemma}
{\bf Proof.} This can be shown similarly as Theorem 8 in \citet{hu2025fast}.

$\hfill \Box$

\begin{lemma}\label{lem:HuLe7}
    Under the completeness assumption \ref{ass:linear}(ii) and feature convergence assumption \ref{ass:linear}(iii), and assume $|\hat{r}_t|\leq M$, for any $\delta>0$, we have
    \begin{align}
        \P\left(\sup_{Q\in\Q}\|\hat{w}_Q\phi-\T Q\|_{\infty} \geq \delta \right) \leq 6d \exp{\left(-\frac{\lambda_0^2}{5174d^2(M+W)^2}n\delta^2\right)}
    \end{align}
\end{lemma}
{\bf Proof.} This can be shown similarly as Lemma 7 in \citet{hu2025fast}.

$\hfill \Box$

\begin{lemma}\label{lem:Cher}
    Consider matrix $\Sigma$ with dimension $d$. Assume $\Sigma$ satisfies
    \begin{align*}
        \bm{X}_k \succcurlyeq \mathbf{0} \quad \text{and} \quad \lambda_{\max}(\bm{X}_k) \leq R \quad \text{almost surely.}
    \end{align*}
    Define $\mu_{\min} := \lambda_{\min}\left( \Sigma \right) \quad \text{and} \quad \mu_{\max} := \lambda_{\max}\left( \Sigma \right)$. Then
    \begin{align*}
        \mathbb{P}\left\{ \lambda_{\min}\left( \Sigma \right) \leq (1 - \delta)\mu_{\min} \right\} &\leq d \cdot \left[ \frac{\mathrm{e}^{-\delta}}{(1 - \delta)^{1 - \delta}} \right]^{\mu_{\min}/R} \quad \text{for } \delta \in [0, 1], \text{ and} \\
        \mathbb{P}\left\{ \lambda_{\max}\left( \Sigma \right) \geq (1 + \delta)\mu_{\max} \right\} &\leq d \cdot \left[ \frac{\mathrm{e}^{\delta}}{(1 + \delta)^{1 + \delta}} \right]^{\mu_{\max}/R} \quad \text{for } \delta \geq 0.
    \end{align*}
\end{lemma}
{\bf Proof.} This can be shown similarly as
matrix Chernoff inequality in \citet{tropp2012user}.

$\hfill \Box$

\begin{lemma}\label{lem:ChLe13}
    For any $Q:\mathcal{S}\times \mathcal{A}\rightarrow \R$, and $\hat\pi=\pi_Q$ be the policy of interest, we have
    \begin{align*}
        J(\pi^*) - J(\pi_{\widehat{Q}}) &\le \sum_{h=1}^{+\infty} \gamma^{h-1} \left( \|Q^* - Q\|_{2,\eta_h^{\hat\pi} \times \pi^*} + \|Q^* - Q\|_{2,\eta_h^{\hat\pi} \times \hat\pi} \right)
    \end{align*}
\end{lemma}
{\bf Proof.} This can be shown similarly as Lemma 13 in \citet{chen2019information}.

$\hfill \Box$

Next, we prove Theorem \ref{thm:regret}.

\medskip \noindent
{\bf Proof of Theorem  \ref{thm:regret}}. 
For the ease of presentation, we use lower letters from now on. Denote $\hat{w}_{Q}$ as the OLS estimator for any $Q\in \Q$. Let $\hat{\Sigma} = \sum_{i=1}^n \phi(x_i, a_i) \phi(x_i, a_i)^\top$ be the empirical design matrix, then $\hat{w}_{Q}$ should be
\begin{align*}
    \hat{w}_{Q} &= \arg\min_{w\in \R^d} \sum_{i=1}^n (w \phi(x_i,a_i) - r_i - \gamma  \max_{a'\in\mathcal{A}} Q(x_i',a'))^2\\
    &= \hat\Sigma^{-1}
    \sum_{i=1}^n \phi(x_i, a_i) \Big(r_i - \gamma  \max_{a'\in\mathcal{A}} Q(x_i',a')\Big).
\end{align*}
And let $\tilde{w}_{Q}$ be the parameter using the estimated reward function $\hat{r}_i$ instead of $y_i$:
\begin{align*}
    \tilde{w}_{Q} = \hat\Sigma^{-1} \sum_{i=1}^n \phi(x_i, a_i) \Big(\hat{r}_i - \gamma  \max_{a'\in\mathcal{A}} Q(x_i',a')\Big).
\end{align*}
Then one-step estimation error can be summarized as
\begin{align*}
    \|\tilde{w}_{Q}^\top \phi - \T Q \|_{\infty} &= \|\tilde{w}_{Q}^\top \phi - \hat{w}_{Q}^\top \phi + \hat{w}_{Q}^\top \phi - \T Q \|_{\infty} \\
    &\leq \|\tilde{w}_{Q}^\top \phi - \hat{w}_{Q}^\top \phi\|_{\infty} + \|\hat{w}_{Q}^\top \phi - \T Q \|_{\infty}\\
    &\leq \|\tilde{w}_{Q} - \hat{w}_{Q} \|_{\infty} \cdot \|\phi\|_{\infty} + \|\hat{w}_{Q}^\top \phi - \T Q \|_{\infty}\\
    &\leq \|\tilde{w}_{Q} - \hat{w}_{Q} \|_{\infty} + \|\hat{w}_{Q}^\top \phi - \T Q \|_{\infty}.
\end{align*}
The second term can be bounded by Lemma \ref{lem:HuLe7}, and the first term captures the estimation error when we estimate the reward. We have
\begin{align*}
    \|\tilde{w}_{Q} - \hat{w}_{Q} \|_{\infty} &= \|\hat\Sigma^{-1}
    \sum_{i=1}^n \phi(x_i, a_i) (\hat{r}_i - r_i ) \|_{\infty} \\
    &\leq \sqrt{d} \|\hat\Sigma^{-1} \|_{2} \cdot \| \sum_{i=1}^n \phi(x_i, a_i) (\hat{r}_i - r_i ) \|_{\infty}.
\end{align*}
For$\|\hat\Sigma^{-1} \|_{2}$, by Lemma \ref{lem:Cher}, we have
\begin{align*}
    \P(\lambda_{\min}(\hat\Sigma)\leq n\lambda_0/2) \leq d \exp{(-n\lambda_0/8)}.
\end{align*}
And when $\lambda_{\min}(\hat\Sigma)\geq n\lambda_0/2$ holds, we have 
\begin{align*}
    \|\Sigma^{-1}\|_2 = 1/\lambda_{\min}(\hat\Sigma) \leq \frac{2}{n\lambda_0}.
\end{align*}
Together with
\begin{align*}
    \| \sum_{i=1}^n \phi(x_i, a_i) (\hat{r}_i - r_i ) \|_{\infty} &\leq \sum_{i=1}^n |\hat{r}_i - r_i| \cdot \|\phi(x_i, a_i)\|_{\infty} \leq n\epsilon_{n},
\end{align*}
Thus when $\lambda_{\min}(\hat\Sigma)\geq n\lambda_0/2$ holds, we have
\begin{align*}
    \|\tilde{w}_{Q} - \hat{w}_{Q}\|_{\infty} \leq \sqrt{d} \frac{2}{n\lambda_0} n\epsilon_{n} = \frac{2\epsilon_{n}\sqrt{d}}{\lambda_0}.
\end{align*}
Then
\begin{align}\label{eq:tildew}
    \P(\sup_{Q}\|\tilde{w}_{Q}^\top\phi - \hat{w}_{Q}^\top\phi\|_{\infty}\geq \frac{2\epsilon_{n}\sqrt{d}}{\lambda_0}) \leq \P(\lambda_{\min}(\hat\Sigma)\leq n\lambda_0/2) \leq d\exp{(-n\lambda_0/8)}.
\end{align}\label{eq:estimationerror}

From Assumption \ref{ass:linear}(iv)  and the fact that $|\HR|\leq 1$, we have $|\hat{r}| \leq (M+\beta)$. Again use Assumption \ref{lem:HuLe7}, and assume there exists a positive constant $C^*$ such that $W < \frac{C^*M}{1-\gamma}$, since the optimal $Q$-function is on the order of $\frac{M}{1-\gamma}$ and $\|\phi\|_\infty$. Assume $\kappa$ is large enough, and use $d<n$, we have
\begin{align}\label{eq:hatw}
    \P\Big(\sup_Q \|\hat{w}_{Q} - \T Q\|_{\infty} \geq \frac{CdM\kappa \log(n)}{(1-\gamma)\lambda_0 \sqrt{n}}\Big) \leq n^{-\kappa},
\end{align}
where $C$ and $\kappa$ are positive constants. Applying Bonferroni inequality to Equation \eqref{eq:tildew} and \eqref{eq:hatw} results in
\begin{align*}
    \P\Big(\sup_Q\|\tilde{w}_{Q} - \T Q\|_{\infty}\leq \frac{2\epsilon_{n}\sqrt{d}}{\lambda_0} + \frac{CdM\kappa \log(n)}{(1-\gamma)\lambda_0 \sqrt{n}} \Big) \leq 1-n^{-\kappa}-d\exp{(-n\lambda_0/8)}.
\end{align*}


Let $\delta_{\text{est}} = \frac{2\epsilon_{n}\sqrt{d}}{\lambda_0}$ and $\delta_{\text{stat}} = \frac{CdM\kappa \log(n)}{(1-\gamma)\lambda_0 \sqrt{n}}$. 
Combining Equation \eqref{eq:tildew} and \eqref{eq:hatw}, with probability at least $1-n^{-\kappa}-d\exp{(-n\lambda_0/8)}$, the one-step approximation error is bounded by:
\begin{align}\label{eq:one_step_bound}
    \sup_Q \|\widehat{Q}_k - \T \widehat{Q}_{k-1}\|_{\infty} \leq \sup_Q \|\tilde{w}_{Q} - \T Q\|_{\infty} \leq \delta_{\text{est}} + \delta_{\text{stat}}.
\end{align}
Substituting \eqref{eq:one_step_bound} into Lemma \ref{lem:HuTh8}, we bound the estimation error of the $Q$-function:
\begin{align*}
    \|Q^* - \widehat{Q}_K\|_{\infty} &\leq \sum_{t=0}^{K-1} \gamma^t (\delta_{\text{est}} + \delta_{\text{stat}}) + \frac{\gamma^K M}{1-\gamma} \\
    &\leq \frac{1}{1-\gamma}(\delta_{\text{est}} + \delta_{\text{stat}}) + \frac{\gamma^K M}{1-\gamma}.
\end{align*}
Finally, we apply Lemma \ref{lem:ChLe13}. Since our errors are bounded in $L_{\infty}$ norm, Lemma \ref{lem:ChLe13} implies the standard performance difference bound (see also \citep{singh1994upper}):
\begin{align*}
    |J(\pi^*) - J(\pi_{\widehat{Q}_K})| \leq \frac{2}{1-\gamma} \|Q^* - \widehat{Q}_K\|_{\infty}.
\end{align*}
Plugging the bound of $\|Q^* - \widehat{Q}_K\|_{\infty}$ into the above inequality yields:
\begin{align*}
    |J(\pi^*) - J(\pi_{\widehat{Q}_K})| &\leq \frac{2}{1-\gamma} \left( \frac{1}{1-\gamma}(\delta_{\text{est}} + \delta_{\text{stat}}) + \frac{\gamma^K M}{1-\gamma} \right) \\
    &= \frac{2(\delta_{\text{est}} + \delta_{\text{stat}})}{(1-\gamma)^2} + \frac{2\gamma^K M}{(1-\gamma)^2}.
\end{align*}
Substituting the expressions for $\delta_{\text{est}}$ and $\delta_{\text{stat}}$ gives the final result.

\hfill $\Box$

\subsection{Proof of Theorem \ref{thm:harm_margin_regret}} \label{app:proof_harm_bound}

\begin{lemma}\label{lem:PDL}
    For any state $x_0\in \mathcal{X}$, we have
    \begin{align*}
        V^{\pi}(x_0) - V^{\pi'}(x_0) = \frac{1}{\gamma} \E_{x\sim d_{x_0}^\pi } \Big[ \E_{a\sim \pi(\cdot|x)}(Q^{\pi'}(x,a) - V^{\pi'}(x)) \Big]
    \end{align*}
\end{lemma}
{\bf Proof. }This is shown by Performance Difference Lemma in \citet{kakade2002approximately}. 

$\hfill \Box$

\medskip \noindent
{\bf Proof of Theorem  \ref{thm:harm_margin_regret}}. 
Let $\mathcal{E}$ be the event where the regret bound in Theorem \ref{thm:regret} holds, and. From Theorem \ref{thm:regret}, we know $\mathbb{P}(\mathcal{E}) \geq 1 - n^{-\kappa}-d\exp{(-n\lambda_0/8)}$.

Conditioned on the event $\mathcal{E}$, we have:
    \begin{align*}
        J(\pi^*) - J(\pi_{\widehat{Q}_K}) \leq \mathcal{R}_{n}.
    \end{align*}
By Lemma \ref{lem:PDL}, the total regret can be expressed as the expected action gap over the state distribution induced by the learned policy. Thus:
    \begin{equation} \label{eq:app_regret_link}
        \mathbb{E}_{x \sim d^{\hat{\pi}}} \left[ \Delta(x) \mathbb{I}(\hat{\pi}(x) \neq \pi^*(x)) \right] = (1-\gamma)(J(\pi^*) - J(\pi_{\widehat{Q}_K})) \leq (1-\gamma)\mathcal{R}_{n}.
    \end{equation}

    Let $\mathbb{P}_{err} = \mathbb{P}_{x \sim d^{\hat{\pi}}}(\hat{\pi}(x) \neq \pi^*(x))$ denote the probability of the learned policy selecting a suboptimal action. We introduce a threshold $t > 0$ to decompose this probability into two parts: states with a small gap and states with a large gap.
    \begin{align*}
        \mathbb{P}_{err} &= \mathbb{P}_{x}( \hat{\pi}(x) \neq \pi^*(x) \cap \Delta(x) \leq t ) + \mathbb{P}_{x}( \hat{\pi}(x) \neq \pi^*(x) \cap \Delta(x) > t ) \\
        &\leq \underbrace{\mathbb{P}_{x}( 0 < \Delta(x) \leq t )}_{T_1} + \underbrace{\mathbb{P}_{x}( \hat{\pi}(x) \neq \pi^*(x) \cap \Delta(x) > t )}_{T_2}.
    \end{align*}
    
    For term $T_1$, we directly apply Assumption \ref{ass:margin} (Margin Condition):
    \[ T_1 \leq C_{\Delta} t^\alpha. \]
    
    For term $T_2$, we note that on the set where $\Delta(x) > t$, the ratio $\frac{\Delta(x)}{t} > 1$. Applying Markov's inequality principle:
    \begin{align*}
        T_2 &\leq \mathbb{E}_{x} \left[ \mathbb{I}(\hat{\pi}(x) \neq \pi^*(x)) \frac{\Delta(x)}{t} \right] \\
        &= \frac{1}{t} \mathbb{E}_{x} [ \Delta(x) \mathbb{I}(\hat{\pi}(x) \neq \pi^*(x)) ].
    \end{align*}
    Using Equation \eqref{eq:app_regret_link}, the numerator is bounded by the regret:
    \[ T_2 \leq \frac{(1-\gamma)\mathcal{R}_{n}}{t}. \]
    
    Combining the bounds, we have $\mathbb{P}_{err} \leq C_{\Delta} t^\alpha + \frac{(1-\gamma)\mathcal{R}_{n}}{t}$. To obtain the tightest bound, we select $t$ such that the two terms are of the same order:
    \[ t^\alpha = \frac{(1-\gamma)\mathcal{R}_{n}}{t} \implies t = \left( (1-\gamma)\mathcal{R}_{n} \right)^{\frac{1}{\alpha+1}}. \]
    Substituting this choice of $t$ back yields:
    \begin{equation}
        \mathbb{P}_{err} \leq (C_{\Delta} + 1) \left( (1-\gamma)\mathcal{R}_{n} \right)^{\frac{\alpha}{\alpha+1}}.
    \end{equation}
    
    Finally, we bound the difference in expected harm. Since $\HR \leq 1$, the value function for the harm component is bounded by $\frac{1}{1-\gamma}$, and the maximum difference is bounded by $\frac{2}{1-\gamma}$. Using the simulation lemma for the harm component:
    \[ 
    | J_{\HR}(\pi^*) - J_{\HR}(\pi_{\widehat{Q}_K}) | \leq \frac{2}{1-\gamma} \mathbb{P}_{err} \leq \frac{2}{1-\gamma} (C_{\Delta} + 1) \left( (1-\gamma)\mathcal{R}_{n} \right)^{\frac{\alpha}{\alpha+1}}.
    \]
    This derivation holds conditioned on $\mathcal{E}$, so the result holds with probability at least $1-n^{-\kappa}-d\exp{(-n\lambda_0/8)}$.

\hfill $\Box$

\subsection{Proof of Corollary \ref{prop:outcome_bound}}

\noindent
{\bf Proof of Corollary \ref{prop:outcome_bound}}. 
    The proof follows directly from the intermediate result established in the proof of Theorem \ref{thm:harm_margin_regret}. 
    As derived in Step 4 of the proof for Theorem \ref{thm:harm_margin_regret}, under the event where the total regret is bounded by $\mathcal{R}_{n}$, we have:
    \begin{equation*}
        \mathbb{P}_{x \sim d^{\hat{\pi}}}(\pi_{\widehat{Q}_K}(x) \neq \pi^*(x)) \leq (C_{\Delta} + 1) \left( (1-\gamma)\mathcal{R}_{n} \right)^{\frac{\alpha}{\alpha+1}}.
    \end{equation*}
    
    Now, consider the value function specifically for the outcome component $Y$. Applying the component-wise Simulation Lemma to $J_Y$:
    \begin{equation*}
        | J_{Y}(\pi^*) - J_{Y}(\pi_{\widehat{Q}_K}) | \leq \frac{1}{1-\gamma} \mathbb{E}_{x \sim d^{\hat{\pi}}} \left[ | Q_{Y}^{\pi^*}(x, \pi^*(x)) - Q_{Y}^{\pi^*}(x, \pi_{\widehat{Q}_K}(x)) | \right].
    \end{equation*}
    Since the outcome $Y$ is bounded by $M$, the maximum possible difference in Q-values is bounded by $\frac{2 M}{1-\gamma}$. The term inside the expectation is non-zero only when $\pi_{\widehat{Q}_K}(x) \neq \pi^*(x)$. Thus:
    \begin{equation*}
        | J_{Y}(\pi^*) - J_{Y}(\pi_{\widehat{Q}_K}) | \leq \frac{2 M}{1-\gamma} \mathbb{P}_{x \sim d^{\hat{\pi}}}(\pi_{\widehat{Q}_K}(x) \neq \pi^*(x)).
    \end{equation*}
    This completes the proof.

\hfill $\Box$

\end{document}